\documentclass[sigconf]{acmart}
\usepackage{booktabs} 

\usepackage{algorithm}
\usepackage{tikz}
\usepackage{algpseudocode}
\usepackage{amssymb}
\usepackage{amsmath}
\usepackage{amsthm}
\usepackage{algorithmicx}
\usepackage{verbatim}
\usepackage{bm}
\usepackage{dsfont}
\usepackage{url}
\usepackage{mathtools}

\usepackage[T1]{fontenc}

\usepackage{aecompl}
\DeclarePairedDelimiter{\ceil}{\lceil}{\rceil}

\newcommand{\defeq}{\vcentcolon=}

\newcommand{\mt}{\mathsf{T}}
\newtheorem{assumption}{Assumption}
\newtheorem{remark}{Remark}

\DeclareMathOperator*{\argmax}{arg\,max}
\DeclareMathOperator*{\argmin}{arg\,min}

\def \btheta {\mathrm{\boldsymbol{\theta}}}

\def \bA {\mathbf{A}}
\def \bx {\mathbf{x}}

\def \cA {\mathcal{A}}
\def \bb {\mathbf{b}}
\def \bI {\mathbf{I}}

\def \cI {\mathcal{I}}
\def \cM {\mathcal{M}}

\def \bbP {\mathbb{P}}
\def \bbE {\mathbb{E}}

\begin{document}

\copyrightyear{2018} 
\acmYear{2018} 
\setcopyright{acmcopyright}
\acmConference[SIGIR '18]{The 41st International ACM SIGIR Conference on Research & Development in Information Retrieval}{July 8--12, 2018}{Ann Arbor, MI, USA}
\acmBooktitle{SIGIR '18: The 41st International ACM SIGIR Conference on Research \& Development in Information Retrieval, July 8--12, 2018, Ann Arbor, MI, USA}
\acmPrice{15.00}
\acmDOI{10.1145/3209978.3210051}
\acmISBN{978-1-4503-5657-2/18/07}
\fancyhead{}
\title{Learning Contextual Bandits in a Non-stationary Environment}

\author{Qingyun Wu, Naveen Iyer, Hongning Wang}
\affiliation{%
  \institution{ Department of Computer Science, University of Virginia }
  \streetaddress{85 Engineer's Way}
  \city{Charlottesville} 
  \state{VA, USA} 
  \postcode{22904}
}
\email{{qw2ky, nki2kd, hw5x}@virginia.edu}

\begin{abstract}

Multi-armed bandit algorithms have become a reference solution for handling the explore/exploit dilemma in recommender systems, and many other important real-world problems, such as display advertisement. However, such algorithms usually assume a stationary reward distribution, which hardly holds in practice as users' preferences are dynamic. This inevitably costs a recommender system consistent suboptimal performance. In this paper, we consider the situation where the underlying distribution of reward remains unchanged over (possibly short) epochs and shifts at unknown time instants. In accordance, we propose a contextual bandit algorithm that detects possible changes of environment based on its reward estimation confidence and updates its arm selection strategy respectively. Rigorous upper regret bound analysis of the proposed algorithm demonstrates its learning effectiveness in such a non-trivial environment. Extensive empirical evaluations on both synthetic and real-world datasets for recommendation confirm its practical utility in a changing environment.
\end{abstract}

%
%



\maketitle

\section{INTRODUCTION}
Multi-armed bandit algorithms provide a principled solution to the explore/exploit dilemma \cite{UCB1,gittins1979bandit,Auer02}, which exists in many important real-world applications such as display advertisement \cite{li2010exploitation}, recommender systems \cite{LinUCB}, and online learning to rank \cite{yue2009interactively}. Intuitively, bandit algorithms adaptively designate a small amount of traffic to collect user feedback in each round while improving their model estimation quality on the fly. In recent years, contextual bandit algorithms \cite{LinUCB,langford2008epoch,filippi2010parametric} have gained increasing attention due to their capability of leveraging contextual information to deliver better personalized online services. They assume the expected reward of each action is determined by a conjecture of unknown bandit parameters and given context, which give them advantages when the space of recommendation is large but the rewards are interrelated.

Most existing stochastic contextual bandit algorithms assume a fixed yet unknown reward mapping function \cite{LinUCB,filippi2010parametric,Li:2016:CFB:2911451.2911548,wu2016contextual,pmlr-v70-gentile17a}. In practice, this translates to the assumption that users' preferences remain static over time. However, this assumption rarely holds in reality as users' preferences can be influenced by various internal or external factors \cite{cialdini1998social}. 
For example, when a sports season ends after a championship, seasonal fans might jump over to following a different sport and not have much interest in the off-season. More importantly, such changes are often not observable to the learners. If a learning algorithm fails to model or recognize the possible changes of the environment, it would constantly make suboptimal choices, e.g., keep making out-of-date recommendations to users.  



 
In this work, moving beyond a restrictive stationary environment assumption, we study a more sophisticate but realistic environment setting where the reward mapping function becomes stochastic over time. More specifically, we focus on the setting where there are abrupt changes in terms of user preferences (e.g., user interest in a recommender system) and those changes are not observable to the learner beforehand. Between consecutive change points, the reward distribution remains stationary yet unknown, i.e., piecewise stationary. Under such a non-stationary environment assumption, we propose a two-level hierarchical bandit algorithm, which  automatically detects and adapts to changes in the environment by maintaining a suite of contextual bandit models during identified stationary periods based on its interactions with the environment.


At the lower level of our hierarchical bandit algorithm, a set of contextual bandit models, referred to as slave bandits, are maintained to estimate the reward distribution in the current environment (i.e., a particular user) since the last detected change point. At the upper level, a master bandit model monitors the `badness' of each slave bandit by examining whether its reward prediction error exceeds its confidence bound. If the environment has not changed, i.e., being stationary since the last change, the probability of observing a large residual from a bandit model learned from that environment is bounded \cite{filippi2010parametric,Improved_Algorithm}. Thus the `badness' of slave bandit models reflects possible changes of the environment. Once a change is detected with high confidence, the master bandit discards the \emph{out-of-date} slave bandits and creates new ones to fit the new environment. Consequentially, the active slave bandit models form an admissible arm set for the master bandit to choose from. At each time, the master bandit algorithm chooses one of the active slave bandits to interact with the user, based on its estimated `badness', and distributes  user feedback to all active slave bandit models attached with this user for model updating. The master bandit model maintains its estimation confidence of the `badness' of those slave bandits so as to recognize the out-of-date ones as early as possible.



We rigorously prove the upper regret bound of our non-stationary contextual bandit algorithm is $O(\Gamma_T \sqrt{S_{\text{max}}} \log S_{\text{max}})$, in which $\Gamma_T$ is the total number of ground-truth environment changes up to time $T$ and $S_{\text{max}}$ is the longest stationary period till time $T$. This arguably is the \emph{lowest} upper regret bound any bandit algorithm can achieve in such a non-stationary environment without further assumptions. Specifically, the best one can do in such an environment is to discard the old model and estimate a new one at each true change point, as no assumption about the change should be made. This leads to the same upper regret bound achieved in our algorithm. However, as the change points are \emph{unknown} to the algorithm ahead of time, any early or late detection of the changes can only result in an increased regret. More importantly, we prove that if an algorithm fails to model the changes a linear regret is inevitable. Extensive empirical evaluations on both a synthetic dataset and three real-world datasets for content recommendation confirmed the improved utility of the proposed algorithm, compared with both state-of-the-art stationary and non-stationary bandit algorithms.

\section{RELATED WORK}



Multi-armed bandit algorithms \cite{UCB1,Gambling,LinUCB,filippi2010parametric, Li:2016:CFB:2911451.2911548, pmlr-v70-gentile17a} have been extensively studied in literature. 
However, most of the stochastic bandit algorithms assume the reward pertaining to an arm is determined by an unknown but \emph{fixed} reward distribution or a context mapping function. This limits the algorithms to a stationary environment assumption, which is restrictive considering the non-stationary nature of many real-world applications of bandit algorithms. 




There are some existing works studying the non-stationary bandit problems. A typical non-stationary environment setting is the abruptly changing environment, or piecewise stationary environment, in which the environment undergoes abrupt changes at unknown time points but remains stationary between two consecutive change points. To deal with such an environment, Hartland et al. \cite{hartland:hal-00113668} proposed the $\gamma-$Restart algorithm, in which a discount factor $\gamma$ is introduced to exponentially decay the effect of past observations. Garivier and Moulines \cite{garivier08_NonStationary} proposed a discounted-UCB algorithm, which is similar to the $\gamma-$Restart algorithm in discounting the historical observations. They also proposed a sliding window UCB algorithm, where only observations inside a sliding window are used to update the bandit model. 
Yu and Mannor \cite{Yu:2009:PBP:1553374.1553524} proposed a windowed mean-shift detection algorithm to detect the potential abrupt changes in the environment. An upper regret bound of $O\big(\Gamma_T \log(T)\big)$ is proved for the proposed algorithm, in which $\Gamma_T$ is the number of ground-truth changes up to time $T$. However, they assume that at each iteration, the agent can query a subset of arms for additional observations. Slivkins and Upfal \cite{adapting-to-a-changing-environment-the-brownian-restless-bandits} considered a continuously changing environment, in which the expected reward of each arm follows Brownian motion. They proposed a UCB-like algorithm, which considers the volatility of each arm in such an environment. The algorithm restarts in a predefined schedule to account for the change of reward distribution. 

Most existing solutions for non-stationary bandit problems focus on context-free scenarios, which cannot utilize the available contextual information for reward modeling. Ghosh et al. proposed an algorithm in \cite{DBLP:journals/corr/GhoshCG17} to deal with environment misspecification in contextual bandit problems. 
Their algorithm comprises a hypothesis test for linearity followed by a decision to use either the learnt linear contextual bandit model or a context-free bandit model. But this algorithm still assumes a stationary environment, i.e., neither the ground-truth linear model nor unknown models are changing over time. Liu et al. \cite{Liu_aaai_2018} proposed to use cumulative sum and Page-Hinkley Test  to detect sudden changes in the environment. An upper regret bound of $O(\sqrt{\Gamma_T T} \log T)$ is proved for one of their proposed algorithms. However, this work is limited to a simplified Bernoulli bandit environment. Recently, Luo et al  \cite{DBLP:journals/corr/abs-1708-01799} studied the non-stationary bandit problem and proposed several bandit algorithms with statistical tests to adapt to changes in the environment. They analyzed various notions of regret including interval regret, switching regret, and dynamic regret.
Hariri et al. \cite{Hariri:2015:AUP:2832747.2832852} proposed a contextual Thompson sampling algorithm with a change detection module, which involves iteratively applying a combination of cumulative sum charts and bootstrapping to capture potential changes of user preference in interactive recommendation. But no theoretical analysis is provided about this proposed algorithm. 

\section{Methodology} \label{sec:method}

We develop a contextual bandit algorithm for a non-stationary environment, where the algorithm automatically detects the changes in the environment and maintains a suite of contextual bandit models for each detected stationary period. 
In the following discussions, we will first describe the notations and our assumptions about the non-stationary environment, then carefully illustrate our developed algorithm and corresponding regret analysis.

\subsection{Problem Setting and Formulation}

In a multi-armed bandit problem, a learner takes turns to interact with the environment, such as a user or a group of users in a recommender system, with a goal of maximizing its accumulated reward collected from the environment over time $T$. At round $t$, the learner makes a choice $a_t$ among a finite, but possibly large, number of arms, i.e.,  $a_t\in\cA=\{a_1, a_2, \dots, a_{K}\}$, and gets the corresponding reward $r_{a_t}$, such as a user clicks on a recommended item. In a contextual bandit setting, each arm $a$ is associated with a feature vector $\bx_a \in \mathbb{R}^d$ ($\lVert \bx_a \rVert_2 \leq 1$ without loss of generality) summarizing the side-information about it at a particular time point. The reward of each arm is assumed to be governed by a conjecture of unknown bandit parameter $\btheta \in \mathbb{R}^d$ ($\lVert {\btheta} \rVert_2 \leq 1$ without loss of generality), which characterizes the environment. 
This can be specified by a reward mapping function $f_{\btheta}$: $r_{a_t} = f_{\btheta}(\bx_{a_t})$. In a stationary environment, $\btheta$ is constant over time. 





In a non-stationary environment, the reward distribution over arms varies over time because of the changes in the environment's bandit parameter $\btheta$. In this paper, we consider abrupt changes in the environment \cite{garivier08_NonStationary,Hariri:2015:AUP:2832747.2832852, hartland:hal-00113668}, i.e., the ground-truth parameter $\btheta$ changes arbitrarily at arbitrary time, but remains constant between any two consecutive change points: 
\begin{equation*}
\!\!\!\!\!\!\underbrace{r_{0}, r_{1}, \!\cdots\!, r_{t_{c_1}-1}}_{ \text{distribute by} ~  f_{\btheta_{c_0}}}, \underbrace{r_{t_{c_1}}, r_{t_{c_1} +1}, \!\cdots\!, r_{t_{c_2}-1}}_{  \text{distribute by} ~ f_{\btheta_{c_1}} }, \!\cdots,\!\underbrace{r_{t_{c_{\Gamma}}}, r_{t_{c_{\Gamma}} +1}, \!\cdots\!, r_{T}}_{ \text{distribute by} ~ f_{\btheta_{c_{\Gamma-1}}} }
\end{equation*}
where the change points $\{t_{c_j}\}_{j=1}^{\Gamma_T -1}$ of the underlying reward distribution and the corresponding bandit parameters $\{\btheta_{c_j}\}_{j=0}^{\Gamma_T -1}$ are unknown to the learner. We only assume there are at most $\Gamma_T -1$ change points in the environment up to time $T$, with $\Gamma_T \ll T$.


To simplify the discussion, linear structure in $f_{\btheta_u}(\bx_{a_t})$ is postulated, but it can be readily extended to more complicated dependency structures, such as generalized linear models \cite{filippi2010parametric}, without changing the design of our algorithm. Specifically, we have,
\begin{equation}
\label{eq:linear_reward}
	r_t =f_{\btheta_t}(\bx_{a_t}) = \bx_{a_t}^\mt \btheta_{t}^* + \eta_t 
\end{equation}
in which $\eta_t$ is Gaussian noise drawn from $N(0,\sigma^2)$, and the superscript $*$ in $\btheta^*_t$ means it is the ground-truth bandit parameter in the environment. In addition, we impose the following assumption about the non-stationary environment, which guarantees the detectability of the changes, and reflects our insight how to detect them on the fly,

\begin{assumption} \label{assumtion:changeAssumption}
 For any two consecutive change points $t_{c_j}$ and $t_{c_{j+1}}$ in the environment, there exists  $\Delta_j >0$, such that when $ t \geq t_{c_{j+1}}$ at least $\rho$ ($ 0 < \rho \leq 1$) portion of all the arms satisfy,
\begin{equation}
\label{eq:assumption}
\lvert \bx_{a_t}^\mt \btheta^*_{t_{c_{j+1}}} - \bx_{a_t}^\mt \btheta^*_{t_{c_{j}}}\rvert > \Delta_j
\end{equation}
\end{assumption}

\begin{remark}
The above assumption is general and mild to satisfy in many practical scenarios, since it only requires a portion of the arms to have recognizable change in their expected rewards. For example, a user may change his/her preference in sports news but not in political news. The arms that do not satisfy Eq \eqref{eq:assumption} can be considered as having small deviations in the generic reward assumption made in Eq \eqref{eq:linear_reward}. We will later prove our  bandit solution remains its regret scaling in the presence of such small deviation. 
\end{remark}


\subsection{Dynamic Linear UCB}

Based on the above assumption about a non-stationary environment, in any stationary period between two consecutive change points, the reward estimation error of a contextual bandit model trained on the observations collected from that period should be bounded with a high probability \cite{Improved_Algorithm,chu2011contextual}. Otherwise, the model's consistent wrong predictions can only come from the change of environment. Based on this insight, we can evaluate whether the stationary assumption holds by monitoring a bandit model's reward prediction quality over time. To reduce variance in the prediction error from one bandit model, we ensemble a set of models, by creating and abandoning them on the fly. 

Specifically, we propose a hierarchical bandit algorithm, in which a master multi-armed bandit model operates over a set of slave contextual bandit models to interact with the changing environment. The master model monitors the slave models' reward estimation error over time, which is referred to as `badness' in this paper, to evaluate whether a slave model is admissible for the current environment. Based on the estimated `badness' of each slave model, the master model dynamically discards \emph{out-of-date} slave models or creates new ones. At each round $t$, the master model selects a slave model with the smallest lower confidence bound (LCB) of `badness' to interact with the environment, i.e., the most promising slave model. The obtained observation $(\bx_{a_t}, r_{a_t})$ is shared across all admissible slave models to update their model parameters. The process is illustrated in Figure \ref{fig:dLinUCB}.

\begin{figure}[t] 
\centering
\includegraphics[width=0.85\linewidth]{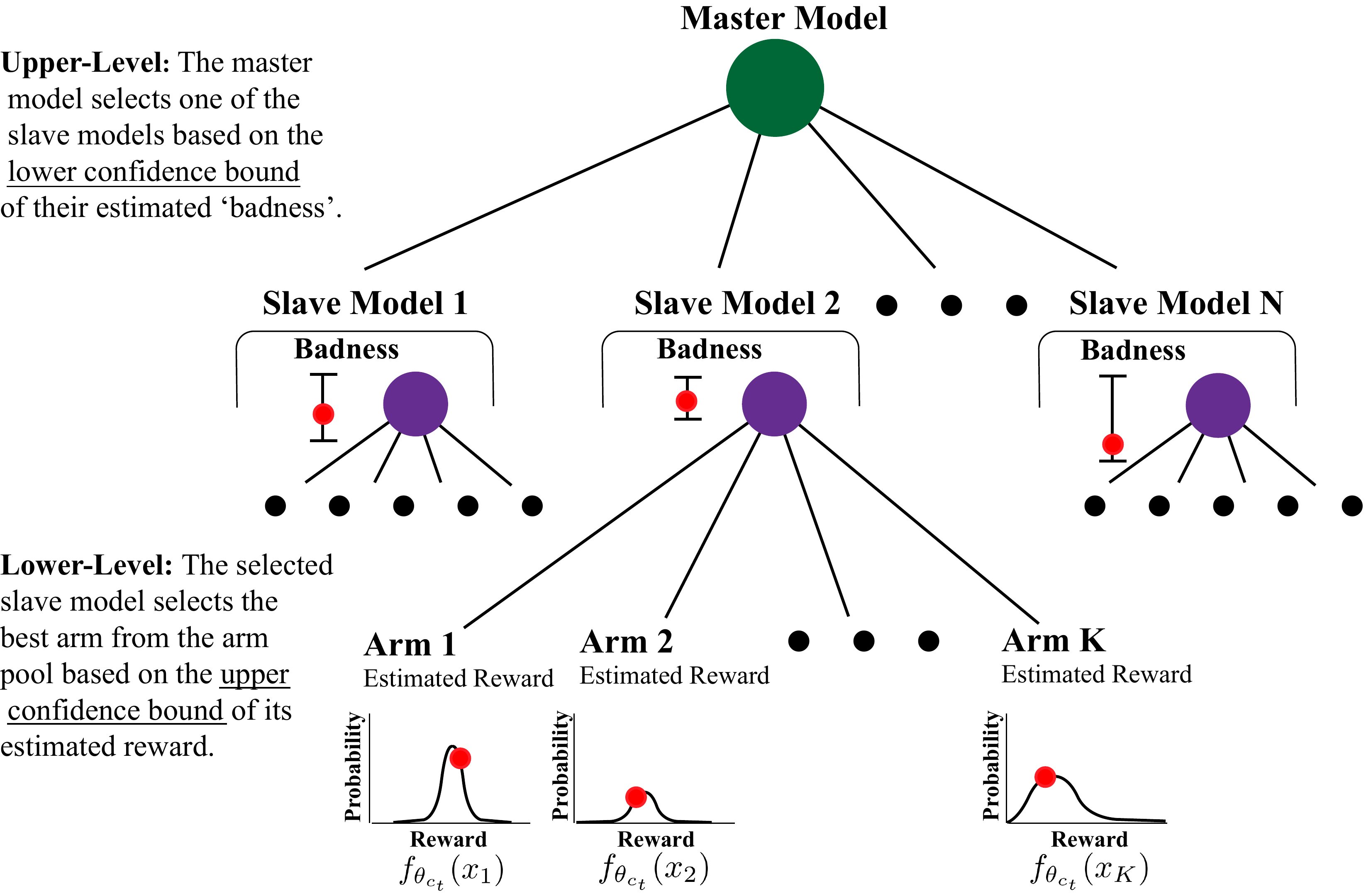}
\caption{Illustration of dLinUCB. The master bandit model maintains the `badness' estimation of slave models over time to detect changes in the environment. At each round, the most promising slave model is chosen to interact with the environment; and the acquired feedback is shared across all admissible slave models for model update.} \label{fig:dLinUCB}
\vspace{-4mm}
\end{figure}


Any contextual bandit algorithm \cite{LinUCB,filippi2010parametric, Li:2016:CFB:2911451.2911548, wu2016contextual} can serve as our slave model. Due to the simplified linear reward assumption made in Eq \eqref{eq:linear_reward}, we choose LinUCB \cite{LinUCB} for the purpose in this paper; but our proposed algorithm can be readily adapted to any other choices of the slave model. This claim is also supported by our later regret analysis.  As a result, we name our algorithm as Dynamic Linear Bandit with Upper Confidence Bound, or dLinUCB in short. 

In the following, we first briefly describe our chosen slave model LinUCB. Then we formally define the concept of `badness', based on which we design the strategy for creating and discarding slave bandit models. Lastly, we explain how dLinUCB selects the most promising slave model from the admissible model set. The detailed description of dLinUCB is provided in Algorithm \ref{alg}.

\noindent\textbf{Slave bandit model: LinUCB.} Each slave LinUCB model maintains all historical observations that the master model has assigned to it. Based on the assigned observations, a slave model $m$ gets an estimate of user preference $\hat \btheta_{t}(m) = \bA_t^{-1}(m) \bb_t(m)$ \cite{LinUCB}, in which $\bA_t(m) = \lambda \bI + \sum_{i \in \cI_{m,t}} \bx_{a_{i}} \bx_{a_{i}}^{\mt} $, $\bI$ is a $d \times d$ identity matrix, $\lambda$ is the coefficient for $L$2 regularization; $\bb_{t}(m) = \sum_{i \in \cI_{m,t}} \bx_{a_{i}} r_{a_{i}}$, and $\cI_{m,t}$ is an index set recording when the observations are assigned to the slave model $m$ up to time $t$. According to \cite{Improved_Algorithm}, with a high probability $1-\delta_1$ the expected reward estimation error of model $m$ is upper bounded: $|\hat r_{a_t}(m) - \bbE[r_{a_t}] |\leq B_t(m, a)$, in which $B_t(m, a) = \big(\sigma^2 \sqrt{d \ln(1+\frac{ |\cI_{m,t}| }{\lambda \delta_1})} + \sqrt{\lambda} \big) \lVert \bx_{a} \rVert_{\bA^{-1}_t(m)}$. Based on the upper confidence bound principle \cite{UCB1}, a slave model $m$ takes an action using the following arm selection strategy (i.e., line 6 in Algorithm \ref{alg}):
\begin{equation} \label{eq:LinUCB_arm_selection}
a_t(m) = \argmax_{a\in\cA} \big( \bx_{a}^{\mt} \hat \btheta_{t}(m) + B_t(m, a) \big)
\end{equation}

\begin{algorithm}[t]
	\caption{Dynamic Linear UCB (dLinUCB)} \label{alg}
	\begin{algorithmic}[1]
		{
			\State \textbf{Inputs:}  $\lambda>0$, $\tau>0$, $\delta_1, \delta_2 \in (0,1)$, $\tilde{\delta}_1 \in [0, \delta_1]$
			\State \textbf{Initialize:} Maintain a set of slave models $\cM_t$ with $\cM_1 = \{ m_1 \}$, initialize $m_1$: $\bA_1(m_1) = \lambda \bI$, $\bb_1(m_1) = \textbf{0}$, $\hat \btheta_1(m_1) = \textbf{0}$; and initialize the `badness' statistics of it: $\hat e_1(m_1) =0$, $d_1(m_1) = 0$
			\For{ $t=1$ to $T$}	    
			\State Choose a slave model from the active slave model set $\tilde m_t = \argmin_{m \in \cM_t} \big( \hat e_t(m) - \sqrt{\ln \tau}\times d_t(m) \big)$
			\State Observe candidate arm pool $\cA_t$, with $\bx_{a}\in \mathbb{R}^d$ for $\forall a\in\cA_t$		
			\State Take action $a_{t}=\argmax_{a\in\cA_t}\big(\bx_{a}^\mt \hat \btheta_t(\tilde m_t)  + B_t(\tilde m_t,a) \big)$, in which $B_t(\tilde m_t,a)$ is defined in Eq \eqref{eq:LinUCB_arm_selection}
			\State Observe payoff $r_{a_{t}}$
			
            \State Set CreatNewFlag = True
            
			\For {$m \in \cM_t$}
			
			\State $e_t(m)=\mathds{1}\{ \lvert \hat r_t(m)-r_t \rvert > B_t(m, a)  + \epsilon$ $ \}$, where $\hat r_t(m) = \bx_{a_t}^\mt \hat \btheta_t(m)$ and $\epsilon =  \sqrt{2} \sigma \text{erf}^{-1}(\delta_1 -1)$
			
			 \If {$ e_t(m) = 0$}
			            \State  Update slave model: $\bA_{t+1}(m) = \bA_t(m) + \bx_{a_t} \bx_{a_t}^{\mt}$,  $\bb_{t+1}(m) = \bb_t(m) + \bx_{a_t}r_{t}$, $\hat \btheta_{t+1} = \bA_{t+1}^{-1}(m)\bb_{t+1}(m)$
			        \EndIf 
			
		 	 \State $\tilde \tau(m) = \text{min}\{ t-t_{m}, \tau \}$, where $t_m$ is when $m$ was created
		 	 \State Update `badness' $\hat e_t(m) = \frac{\sum_{i=t-\tilde \tau}^t e_i(m)}{ \tilde \tau(m)}$, $d_t(m) = \sqrt{\frac{\ln 1/\delta_2}{2\tilde \tau(m)}}$
		 	 
		 	  \If{$\hat e_t(m) < \tilde{\delta_1} + d_t(m)$ }
			    \State Set CreatNewFlag = False
		 	 \ElsIf{$\hat e_t(m) \geq \delta_1  + d_t(m)$}
			\State Discard slave model $m$: $M_{t+1} = M_{t} - m$
			\EndIf  
			\EndFor 
           
        \If {\text{CreateNewFlag} or $\cM_t  = \emptyset $}
			\State Create a new slave model $m_t$:  $\cM_{t+1} = \cM_t + m_t$
			\State Initialize $m_t$: $\bA_t(m_t) = \lambda \bI$, $\bb_t(m_t) = \textbf{0}$, $\hat \btheta_t(m_t) = \textbf{0}$
			\State Initialize `badness' statistics of $m_t$: $\hat e_t(m_t) =0$, $d_t(m_t) = 0$

		\EndIf
			
	\EndFor
		}
	\end{algorithmic}
\end{algorithm}

\noindent\textbf{Slave model creation and abandonment.} For each slave bandit model $m$, we define a binary random variable $e_i(m)$ to indicate whether the slave model $m$'s prediction error at time $i$ exceeds its confidence bound,
\begin{equation}
    e_i(m) \defeq \mathds{1}\big\{ \lvert \hat r_i(m) -r_i(m) \rvert > B_i(m,a_i) + \epsilon \big\}
\end{equation}
where $\epsilon = \sqrt{2}\sigma\text{erf}^{-1}(\delta_1 -1)$ and $\text{erf}^{-1}(\cdot)$ is the inverse of Gauss error function. $\epsilon$ represents the high probability bound of Gaussian noise in the received feedback.

According to Eq \eqref{eq:slave_residual_bound} in Theorem \ref{theorem:slave_CB}, if the environment stays stationary since the slave model $m$ has been created, we have $\bbP(e_{i}(m)=1) \leq \delta_1$, where $\delta_1 \in (0,1)$ is a hyper-parameter in $B_i(m,a)$. Therefore, if we observe a sequence of consistent prediction errors from the slave model $m$, it strongly suggests a change of environment, so that this slave model should be abandoned from the admissible set. Moreover, 
we introduce a size-$\tau$ sliding window to only accumulate the most recent observations when estimating the expected error in slave model $m$. The benefit of sliding window design will be discussed with more details later in Section \ref{sec:regret}. 

We define $\hat e_t(m) \defeq \frac{\sum_{i=t- \tilde \tau(m)}^t e_i(m)}{\tilde \tau(m)}$, which estimates the `badness' of  slave model $m$ within the most recent period $\tilde \tau$ to time $t$, i.e., $\tilde \tau(m) = \text{min}\{ t-t_{m}, \tau \}$, in which $t_m$ is when model $m$ was created. Combining the concentration inequality in Theorem \ref{theorem:chernoffBound} (provided in the appendix), we have the assertion that if in the period $[t-\tilde \tau(m), t]$ the stationary hypothesis is true, for any given $\delta_1 \in (0,1)$ and $\delta_2 \in (0,1)$, with a probability at least $1-\delta_2$, the expected `badness' of slave model $m$ satisfies,
\small
\begin{equation} \label{eq:accumulated_model_error}
\!\!\!\hat e_t(m) \leq \bbE[e_t(m)] + \sqrt{ \frac{\ln (1/\delta_2) }{2 \tilde \tau(m) } }  \leq \delta_1 +  \sqrt{ \frac{\ln (1/\delta_2)}{2\tilde \tau(m)} }
\end{equation}
\normalsize

Eq \eqref{eq:accumulated_model_error} provides a tight bound to detect changes in the environment. If the environment is unchanged, within a sliding window the estimation error made by an up-to-date slave model should not exceed the right-hand side of Eq \eqref{eq:accumulated_model_error} with a high probability. Otherwise, the stationary hypothesis has to be rejected and thus the slave model $m$ should be discarded. Accordingly, if none of the slave models in the admissible bandit set satisfy this condition, a new slave bandit model should be created for this new environment. Specifically, the master bandit model controls the slave model creation and abandonment in the following way.

\noindent$\bullet$ \textit{Model abandonment}: when the  slave model $m$'s estimated `badness' exceeds its upper confidence bound defined in Eq \eqref{eq:accumulated_model_error}, i.e., $\hat e_t(m) > \delta_1 + \sqrt{ \frac{\ln (1/\delta_2)}{2  \tilde \tau(m)} }$, it will be discarded and removed from the admissible slave model set. This corresponds to line 18-20 in Algorithm \ref{alg}.

\noindent$\bullet$ \textit{Model creation}: When no slave model's estimated `badness' is within its expected confidence bound, i.e., no slave model satisfies $\hat e_t(m) \leq \tilde{\delta_1} + \sqrt{ \frac{\ln (1/\delta_2)}{2 \tilde \tau(m)} }$, a new slave model will be created. $\tilde{\delta_1} \in [0, \delta_1]$ is a parameter to control the sensitivity of dLinUCB, which affects the number of maintained slave models. When $\tilde{\delta_1} = \delta_1$, the threshold of creating and abandoning a slave model matches and the algorithm only maintains one admissible slave model. When $\tilde{\delta_1} < \delta_1$ multiple slave models will be maintained. The intuition is that an environment change is very likely to happen when all active slave models face a high risk of being out-of-date (although they have not been abandoned yet). This corresponds to line 8, 16-17, and 22-26 in Algorithm \ref{alg}.


\noindent\textbf{Slave model selection and update.} At each round, the master bandit model selects one active slave bandit model to interact with the environment, and updates all active slave models with the acquired feedback accordingly. As we mentioned before, with the model abandonment mechanism every active slave model is guaranteed to be admissible for taking acceptable actions; but they are associated with different levels of risk of being out of date. A well-designed model selection strategy can further reduce the overall regret, by minimizing this risk. Intuitively, when facing a changing environment, one should prefer a slave model with the lowest empirical error in the most recent period.   


The uncertainty in assessing each slave model's `badness' introduces another explore-exploit dilemma, when choosing the active slave models. Essentially, we prefer a slave model of lower `badness' with a higher confidence. We realize this criterion by selecting a slave model according to its Lower Confidence Bound (LCB) of the estimated `badness.' This corresponds to line 4 in Algorithm \ref{alg}.


Once the feedback $(\bx_{a_t}, r_t)$ is obtained from the environment on the selected arm $a_t$, the master algorithm can not only update the selected slave model but also all other active ones for both of their `badness' estimation and model parameters (line 11-13 and line 15 in Algorithm \ref{alg} accordingly). This would reduce the sample complexity in each slave model's estimation. However, at this stage, it is important to differentiate those ``about to be out-of-date'' models from the ``up-to-date'' ones, as any unnecessary model update blurs the boundary between them. As a result, we only update the \emph{perfect} slave models, i.e., those whose `badness' is still zero at this round of interaction; and later we will prove this updating strategy is helpful to decrease the chance of late detection.




\subsection{Regret Analysis}
\label{sec:regret}
In this section, we provide a detailed regret analysis of our proposed dLinUCB algorithm. We focus on the accumulated pseudo regret, which is formally defined as,
\begin{equation} \label{eq:regretDefinition}
  \textbf{R}(T) = \sum_{t=1}^T \big(\bbE[r_{a_t^*}] -\bbE[r_{a_t}]\big)
\end{equation}
where $a_t^*$ is the best arm to select according to the oracle of this problem, and $a_t$ is the arm selected by the algorithm to be evaluated. 

It is easy to prove that if a bandit algorithm does not model the change of environment, it would suffer from a linearly increasing regret:  An optimal arm in the previous stationary period may become sub-optimal  after the change; but the algorithm that does not model environment change will constantly choose this sub-optimal arm until its estimated reward falls behind the other arms'. This leads to a linearly increasing regret in each new stationary period.

Next, we first characterize the confidence bound of reward estimation in a linear bandit model in Theorem \ref{theorem:slave_CB}. Then we prove the upper regret bound of two variants of our dLinUCB algorithm in Theorem \ref{theorem:regretBoundOnly1} and Theorem \ref{theorem:regret_bound}. More detailed proofs are provided in the appendix. 

\begin{theorem} \label{theorem:slave_CB}
 For a linear bandit model $m$ specified in Algorithm \ref{alg}, if the underlying environment is stationary, for any $\delta_1 \in (0,1)$ we have the following inequality with probability at least $1-\delta_1$, 	
	\begin{equation} \label{eq:slave_residual_bound}
	\lvert \hat r_t(m) -  r_t \rvert     \leq B_t(m,a) + \epsilon 
	\end{equation}
	where $B_t(m,a) =\alpha_t \lVert \bx_{a_t} \rVert_{\bA_{t-1}^{-1}} $ with $\alpha_t =  \big(\sigma^2 \sqrt{d \ln(1+\frac{ |\cI_t(m)| }{\lambda \delta_1})} + \sqrt{\lambda} \big)$,  $\epsilon = \sqrt{2}\sigma \text{erf}^{-1}(\delta_1 -1)$, $\sigma$ is the standard deviation of the Gaussian noise in reward feedback, and $\text{erf}(\cdot)$ is the Gauss error function.
\end{theorem}

Denote $R_{\text{Lin}}(S)$ as the upper regret bound of a linear bandit model within a stationary period $S$. Based on Theorem \ref{theorem:slave_CB}, one can prove that $R_{\text{Lin}}(S) \leq \sqrt{dS\log (\lambda+\frac{S}{d})} \Big(  \sigma^2 \sqrt{d\log (1+ \frac{S}{\lambda \delta_1}}) + \sqrt{\lambda} \Big)  $  \cite{Improved_Algorithm}. In the following, we provide an upper regret bound analysis for the basic version of dLinUCB, in which the size of the admissible slave models is restricted to one (i.e., by setting $\tilde{\delta_1} = \delta_1$).  

\begin{theorem} \label{theorem:regretBoundOnly1}
When Assumption \ref{assumtion:changeAssumption} is satisfied with $\Delta \geq 2 \sqrt{\lambda} + 2\epsilon$, if $\delta_1$ and $\tau$ in Algorithm \ref{alg} are set according to Lemma \ref{lemma:late_detection}, and $\delta_2$ is set to $\delta_2 \leq \frac{1}{2S_{\text{max}}}$, with probability at least $(1-\delta_1)(1-\delta_2)(1-\frac{\delta_2}{1- \delta_2} )$, the accumulated regret of dLinUCB satisfies,
\begin{equation}
\label{eq:regret-one-slave}
\textbf{R}(T) \leq 2 \Gamma_T  R_{\text{Lin}}(S_{\text{max}}) + \Gamma_T (\tau + \frac{4}{1-\delta_2})
\end{equation}
where $S_{\text{max}}$ is the length of the longest stationary period up to $T$.
\end{theorem}
\begin{proof}

\textbf{Step 1:} If change points can be perfectly detected, the regret of dLinUCB can be bounded by $ \sum_{j=0}^{\Gamma_T-1}  R_{\text{Lin}}(S_{c_j}) $. However, additional regret may accumulate if early or late detection happens. In the following two steps, we will bound the possible additional regret from early detection, denoted as $R_{\text{early}}$, and that from late detection, denoted as $R_{\text{late}}$. 

\textbf{Step 2:} 
Define $k_{c_j}$ as the number of early detection within this stationary period $[t_{c_{j}}, t_{c_{j+1}}]$, with $S_{c_j} = t_{c_{j+1}} - t_{c_j} $.  
Define $p_e$ as the probability of early detection in the stationary period, we have $\bbP(k_{c_j} = k)= \binom{S_{c_j}}{k} p_e^{k}(1-p_e)^{S_{c_j}-k}$. According to Lemma \ref{lemma:early_detection}, we have $p_e \leq \delta_2$. Combining the property of binomial distribution $B(S_{c_j}, p_e)$ and Chebyshev's concentration inequality, we have $k_{c_j} \leq 2S_{c_{j}} \delta_2 $ with probability at least $1-\frac{1-\delta_2 }{2S_{c_j} \delta_2 }$. Hence, with a probability $(1-\delta_3)\times(1-\delta_1)^{k_{S_{\text{max}}}}$, we have
$R_{\text{early}} \leq \sum_{j=1}^{\Gamma_T} k_{c_{j}} R_{\text{Lin}}(\frac{s_{c_j}}{k_{c_j}}) \leq \sum_{j=1}^{\Gamma_T} 2S_{c_{j}} \delta_2 R_{\text{Lin}}(\frac{1}{2 \delta_2}) \leq 2\delta_2 \Gamma_T S_{T} R_{\text{Lin}}(\frac{1}{2  \delta_2})$. Considering the calculation of $R_{\text{Lin}}$, when $\delta_2 \leq  \frac{1}{2S_{\text{max}}}$ we have $ R_{\text{early}} \leq  \Gamma_T R_{\text{Lin}}(S_{\text{max}})$ with a probability at least $(1-\delta_2)(1-\delta_1)$. This upper bounds the additional regret from any possible early detection, and maintains it in the same order as the slave model's. 

\textbf{Step 3}. 
Define $\tilde k_{c_j}$ as the number of interactions where the environment has changed (comparing to $\btheta^*_{c_j}$) but the change is not detected by the algorithm. The additional regret from this late detection can be bounded by $2\tilde k_{c_j}$ (i.e., the maximum regret in each round of interaction). Define $p_d$ as the probability of detection after the change happens, we have $\bbP(\tilde k_{c_j} = k) = (1-p_d)^{k-1}p_d$, i.e., a Geometric distribution. According to Lemma \ref{lemma:late_detection}, $p_d \geq 1-\delta_2$. Based on the property of Geometric distribution $G(p_d)$ and Chebyshev's inequality, we have $\tilde k_{c_j} \leq \frac{2}{1-\delta_2}$ with probability $1-\frac{\delta_2}{1- \delta_2}$. If we consider the case where the change point locates inside the sliding window $\tau$, we may have at most another $\tau$ delay after each change point. Therefore, the additional regret from late detection can be bounded by $R_{\text{late}} \leq \Gamma_T \big(\tau + \frac{4}{1-\delta_2}  \big)$, which is not directly related to the length of any stationary period. 

Combining the above three steps concludes the proof. 
\end{proof}

\begin{lemma}[Bound the probability of early detection] \label{lemma:early_detection}
For $\delta_2 \in (0,1)$ and any slave model in Algorithm \ref{alg},
\begin{equation*}
   p_e = \bbP[\hat e_t(m) > \delta_1 + d_t(m) | \text{stationary in past} ~ \tilde{\tau}(m) \text{ rounds}] \leq \delta_2.
\end{equation*}
\end{lemma}
The intuition behind Lemma \ref{lemma:early_detection} is that when the environment is stationary, the `badness' of a slave model $m$ should be small and bounded according to Eq \eqref{eq:accumulated_model_error}. 

\begin{lemma}[Bound the probability of late detection] \label{lemma:late_detection}
When the magnitude of change in the environment in Assumption \ref{assumtion:changeAssumption} satisfies $\Delta > 2\sqrt{\lambda} + 2\epsilon $, and the shortest stationary period length $S_{min}$ satisfies $S_{\text{min}} > \frac{\sqrt{\lambda}}{2\rho}(\Delta - 2\sqrt{\lambda} - 2\epsilon )$, for any $\delta_2 \in (0,1)$, if $\delta_1$ and $\tau$ in Algorithm \ref{alg} are set to $\delta_1 \leq 1- \frac{1}{\rho}\big(1-\frac{\sqrt{\lambda}}{2S_{\text{min}}\rho}(\Delta - 2\sqrt{\lambda} - 2\epsilon )\big)$ and $\tau \geq \frac{2\ln \frac{2}{\delta_2 }}{ (\rho (1-\delta_1) - \delta_1)^2}$, for any slave model $m$ in Algorithm \ref{alg}, we have,
\begin{equation*}
p_d = \bbP\big(\hat e_t(m) > \delta_1 + d_t(m)|  \text{ changed within past} ~\tilde{\tau}(m)\big) \geq 1- \delta_2.
\end{equation*}
\end{lemma}

The intuition behind Lemma \ref{lemma:late_detection} is that when the environment has changed, with a high probability that Eq \eqref{eq:slave_residual_bound} will not be satisfied in an out-of-date slave model. It means that we will accumulate larger badness from this slave model. In both Lemma \ref{lemma:early_detection} and \ref{lemma:late_detection}, $\delta_2$ is a parameter controlling the confidence of the `badenss' estimation in Chernoff Bound; and therefore an input to the algorithm.

\begin{remark}[How the environment assumption affects dLinUCB]
1. The magnitude of environment change $\Delta$ affects whether a change is detectable by our algorithm. 
However, we need to emphasize that when $\Delta$ is very small, the additional regret from re-using an out-of-date slave model is also small. In this case, a similar scale of regret bound can still be achieved, which will be briefly proved in Appendix and empirically studied in \emph{Section \ref{exp_simulation}}. 2. We require the shortest stationary period length  $S_{\text{min}} > \max\{ \frac{\sqrt{\lambda}}{2\rho}(\Delta - 2\sqrt{\lambda} - 2\epsilon ), \tau \}$, which guarantees there are enough observations accumulated in a slave model to make an informed model selection. 3. The portion of changed arms $\rho$ will affect the probability of achieving our derived regret bound, as we require  $\delta_1 \leq 1- \frac{1}{\rho}\big(1-\frac{\sqrt{\lambda}}{2S_{\text{min}}\rho}(\Delta - 2\sqrt{\lambda} - 2\epsilon )\big)$. $\rho$ also interacts with $S_{\text{min}}$ and $\tau$: when $\rho$ is small, more observations are needed for a slave model to detect the changes. The effect of $\rho$ and $S_{\text{min}}$ will also be studied in our empirical evaluations.
\end{remark}

Theorem \ref{theorem:regretBoundOnly1} indicates with our model update and abandonment mechanism, each slave model in dLinUCB is `admissible' in terms of upper regret bound. In the following, we further prove that maintaining multiple slave models and selecting them according to their LCB of `badness' can further improve the regret bound.


\begin{theorem} \label{theorem:regret_bound}
Under the same condition as specified in Theorem \ref{theorem:regretBoundOnly1}, with probability at least $(1-\delta_1)(1-\delta_2)(1-\frac{\delta_2}{1- \delta_2} )$, the expected accumulated regret of dLinUCB up to time $T$ can be bounded by,
\begin{align} \label{eq:final_regret}
\textbf{R}(T) \leq & \big (\sum_{j=0}^{\Gamma_T -1} R_{\text{Lin}}(S_{c_j}) + 2 \Gamma_T ( \tau + \frac{4}{1-\delta_2})  \\ \nonumber
& + \sum_{j=0}^{\Gamma_T-1} \big (8 \sum_{m \in \cM, m \neq m_{{c_j}}^*} \frac{\ln S_{c_j}}{g_{m,m^*_{c_j}}} + (1+\frac{\pi^2}{3}) g_{ m,m_{c_j}^*} \big)
\end{align} 
in which $m^*_{c_j}$ is the best slave model among all the active ones in the stationary period $[t_{c_j}, t_{c_{j+1}}]$ according to the oracle, and $g_{m,m_{c_j}^*}$ is difference between the accumulated expected reward from the selected model $m$ and that from $m_{c_j}^*$ in the period $[t_{c_j}, t_{c_{j+1}}]$ .
\end{theorem}

\begin{proof}
Define the optimal expected cumulative reward in the stationary period $[t_{c_j}, t_{c_j+1}]$ according to the oracle as $G^*(S_{c_j})$ and the expected accumulative reward in dLinUCB as $G(S_{c_j})$. $G_{m_{c_j}^*}(S_{c_j})$ is the expected cumulative reward from $m_{c_j}^*$.  
The accumulated regret of dLinUCB can be written as,
\begin{equation} 
\label{eq:total_regret}
\textbf{R}(T) = \sum_{j=0}^{\Gamma_T -1} \big( G^*(S_{c_j}) - G_{m_{c_j}^*}(S_{c_j})\big) + \big( G_{m_{c_j}^*}(S_{c_j})- G(S_{c_j}) \big)
\end{equation}
The first term of Eq\eqref{eq:total_regret} can be bounded based on Theorem \ref{theorem:regretBoundOnly1}. 
Define $\tilde N_{c_j}(m)$ as the number of times a slave model $m$ is selected when it is not the best in $[t_{c_j}, t_{c_j+1}]$: $\tilde N_{c_j}(m) = \sum_{t={t_{c_j}}}^{t_{c_j+1}} \mathds{1}\{\tilde m_t = m, m_{c_j}^* \neq m\}$, we have $G_{m_{c_j}^*}(S_{c_j})- G(S_{c_j}) \leq \sum_{m \in \cM } g_{m, m^*_{c_j}}\bbE[\tilde N_{c_j}(m)]$.
In Lemma \ref{lemma:N_bound}, we provide the bound of $\bbE[\tilde N_{c_j}(m)]$. Substituting the above conclusions into Eq \eqref{eq:total_regret} finishes the proof.
\end{proof}

\begin{lemma} \label{lemma:N_bound}
The model selection strategy in Algorithm \ref{alg} guarantees,
\begin{equation}
\bbE[\tilde N_{c_j}(m)] \leq \frac{8\ln S_{c_j}}{g_{m,m^*_{c_j}}^2} + 1 + \frac{\pi^2}{3}
\end{equation}
\end{lemma}

\begin{remark}[regret comparison of dLinUCB with one slave model and multiple slave models] 
By maintaining multiple admissible slave models and selecting one according to the LCB of `badness' when interacting with the environment, dLinUCB achieves a regret reduction in the first part of Eq \eqref{eq:final_regret}. Although there is additional regret introduced by switching between the best model $m^*$ and the chosen model $\tilde m$, this added regret increases much slower than that resulted from any slave model (i.e., $\ln T$ v.s., $R_{\text{Lin}}(T)$); and thus maintaining multiple slave models is always beneficial. Besides, the order of upper regret bound of dLinUCB in both cases is $O(\Gamma_T \sqrt{S_{\text{max}}} \log S_{\text{max}})$, which is the best upper regret bound a bandit algorithm can achieve in such a non-stationary environment \cite{garivier08_NonStationary}, and it matches the lower bound up to a $\Gamma_T \log S_{\text{max}}$ factor.
\end{remark}

\begin{remark}[Generalization of dLinUCB]
Our theoretical analysis confirms that any contextual bandit algorithm can be used as the slave model in dLinUCB, as long as the its reward estimation error is bounded with a high probability, which corresponds $B_t(m, a)$ in Eq \eqref{eq:slave_residual_bound}. The overall regret of dLinUCB will only be a factor of the actual number of changes in the environment, which is arguably inevitable without further assumptions about the environment. 
\end{remark}

\section{Evaluations}

\begin{table*}[t]
\centering
\caption{Accumulated regret with different noise level $\sigma$, environment change $\Delta$ and stationary period length $S$.}  \label{tab:sim_table}
\vspace{-3mm}
\begin{tabular}{|ccccccc|} \hline
($\sigma$, $\Delta$, $S$)& (0.1, 0.9, 800) & (0.05, 0.9, 800) & (0.01, 0.9, 800) &(0.01, 0.5, 800) &  (0.01, 0.1, 800) & (0.01, 0.9, 400) \\
 \hline
 \textbf{dLinUCB }& \textbf{ 87.46 $\pm$ 3.61 } & \textbf{65.94$\pm$ 2.30} & \textbf{54.07$\pm$ 3.95} & \textbf{44.94 $\pm$ 2.90} & \textbf{46.12 $\pm$ 4.63}  & \textbf{111.72 $\pm$ 4.87 } \\  \
adTS & 360.75$\pm$ 39.59 & 249.63 $\pm$ 27.26 & 207.95$\pm$ 22.28 & 189.07$\pm$18.39 & 177.55.$\pm$20.36 & 412.55 $\pm$ 14.53 \\ 
LinUCB & 436.84 $\pm$ 40.23 & 386.10$\pm$21.88 &  347.19$\pm$ 14.95 & 264.87$\pm$ 21.53 & 226.87$\pm$ 32.15 & 405.82 $\pm$ 33.38    \\ 
Meta-Bandit & 1822.31$\pm$ 80.67 & 1340.01$\pm$29.94 & 1354.03 $\pm$ 22.29 & 1329.51 $\pm$ 18.93 & 1402.63 $\pm$ 24.85 & 1388.81 $\pm$ 115.91  \\ 
WMDUCB1 & { 2219.36$\pm$ 142.16  }& {1652.99$\pm$ 21.33} & 1635.35 $\pm$ 73.96 & 1464.11 $\pm$ 89.16 & 1506.55 $\pm$ 41.52 & 1691.75 $\pm$ 48.09 \\ 
\hline
\end{tabular}
\end{table*}

We performed extensive empirical evaluations of dLinUCB against several related baseline algorithms, including: 1) the state-of-the-art contextual bandit algorithm LinUCB \cite{LinUCB}; 2) adaptive Thompson Sampling algorithm \cite{Hariri:2015:AUP:2832747.2832852} (named as adTS)  which has a change detection module; 3) windowed mean-shift detection algorithm \cite{Yu:2009:PBP:1553374.1553524} (named as WMDUCB1), which is a UCB1-type algorithm with a change detection module ; and 4) Meta-Bandit algorithm \cite{hartland:hal-00113668}, which switches between two UCB1 models. 

\subsection{Experiments on synthetic datasets}
\label{exp_simulation}
In simulation, we generate a size-$K$ ($K = 1000$) arm pool $\cA$, in which each arm $a$ is associated with a $d$-dimensional feature vector $\bx_a \in \mathbb{R}^d$ with $\lVert \bx_{a}\rVert_2 \leq 1$. Similarly, we create the ground-truth bandit parameter $\btheta^* \in \mathbb{R}^d$ with $\lVert \btheta^* \rVert_2 \leq 1$, which is not disclosed to the learners. Each dimension of $\bx_a$ and $\btheta^*$ is drawn from a uniform distribution $U(0,1)$. 
At each round $t$, only a subset of arms in $\cA$ are disclosed to the learner for selection, e.g., randomly sample 10 arms from $\cA$ without replacement. 
The ground-truth reward $r_{a}$ is corrupted by Gaussian noise $\eta \sim N(0,\sigma^2)$ before being fed back to the learner. The standard deviation of Gaussian noise $\sigma$ is set to 0.05 by default. To make the comparison fair, 
at each round $t$, the same set of arms are presented to all the algorithms being evaluated. To simulate an abruptly changing environment, after every $S$ rounds, we randomize $\btheta^*$ with respect to the constraint $|\bx_{a}^{\mt} \btheta^*_{t_{c_j}} - \bx_{a}^{\mt} \btheta^*_{t_{c_{j+1}}}| > \Delta_j $ for $\rho$ proportion of arms in $\cA$. We set $\lambda$ to 0.1, $S$ to 800 and $\Delta$ to $0.9$ by default.

Under this simulation setting, all algorithms are executed to 5000 iterations and the parameter $\tau$ in dLinUCB is set to 200. Accumulated regret defined in Eq \eqref{eq:regretDefinition} is used to evaluate different algorithms and is reported in Figure \ref{fig:simulation}. The bad performance of LinUCB illustrates the necessity of modeling the non-stationarity of the environment -- its regret only converges in the first stationary period, and it suffers from an almost linearly increasing regret, which is expected according to our theoretical analysis in Section \ref{sec:regret}. adTS is able to detect and react to the changes in the environment, but it is slow in doing so and therefore suffers from a linear regret at the beginning of each stationary period before converging. dLinUCB, on the other hand, can quickly identify the changes and create corresponding slave models to capture the new reward distributions, which makes the regret of dLinUCB converge much faster in each detected stationary period. In Figure \ref{fig:simulation} we use the black and blue vertical lines to indicate the actual change points and the detected ones by dLinUCB respectively. It is clear that dLinUCB detects the changes almost immediately every time. WMDUCB1 and Meta-Bandit are also compared, but since they are context-free bandits, they performed much worse than the above contextual bandits. To improve visibility of the result, we exclude them from Figure \ref{fig:simulation} and instead report their performance in Table \ref{tab:sim_table}.

As proved in our regret analysis, dLinUCB's performance depends the magnitude of change $\Delta$ between two consecutive stationary periods, the Gaussian noise $\sigma$ in the feedback, and the length $S$ of stationary period. In order to investigate how these factors affect dLinUCB, we varied these three factors in simulation. We ran all the algorithms for 10 times and report the mean and standard deviation of obtained regret in Table \ref{tab:sim_table}. In all of our environment settings, dLinUCB consistently achieved the best performance against all baselines. In particular, we can notice that the length $S$ of stationary period plays an important role in affecting dLinUCB's regret (and also in adTS). This is expected from our regret analysis: since $T$ is fixed, a smaller $S$ leads to a larger $\Gamma_T$, which linearly scales dLinUCB's regret in Eq \eqref{eq:regret-one-slave} and \eqref{eq:final_regret}. A smaller noise level $\sigma$ leads to reduced regret in dLinUCB, as it makes the change detection easier. 
Last but not least, the magnitude of change $\Delta$ does not affect dLinUCB: when $\Delta$ is large, the change is easy to detect; when $\Delta$ is small, the difference between two consecutive reward distributions is small, and thus the added regret from an out-of-date slave model is also small. Again the context-free algorithms WMDUCB1 and Meta-Bandit performed much worse than those contextual bandit algorithms in all the experiments. 

In addition, we also studied the effect of $\rho$ in dLinUCB by varying $\rho$ from 0.0 to 1.0. dLinUCB achieved the lowest regret when $\rho = 0$, since the environment becomes stationary. When $\rho > 0$: dLinUCB achieves the best regret (with regret of 54.07 $\pm $ 3.95) when $\rho = 1.0$, however as $\rho$ becomes smaller the regret is not affected too much (with regret of 57.59 $\pm$ 3.44). These results further validate our theoretical regret analysis and unveil the nature of dLinUCB in a piecewise stationary environment.


\begin{figure}[t] 
\centering
\includegraphics[width=0.92\linewidth]{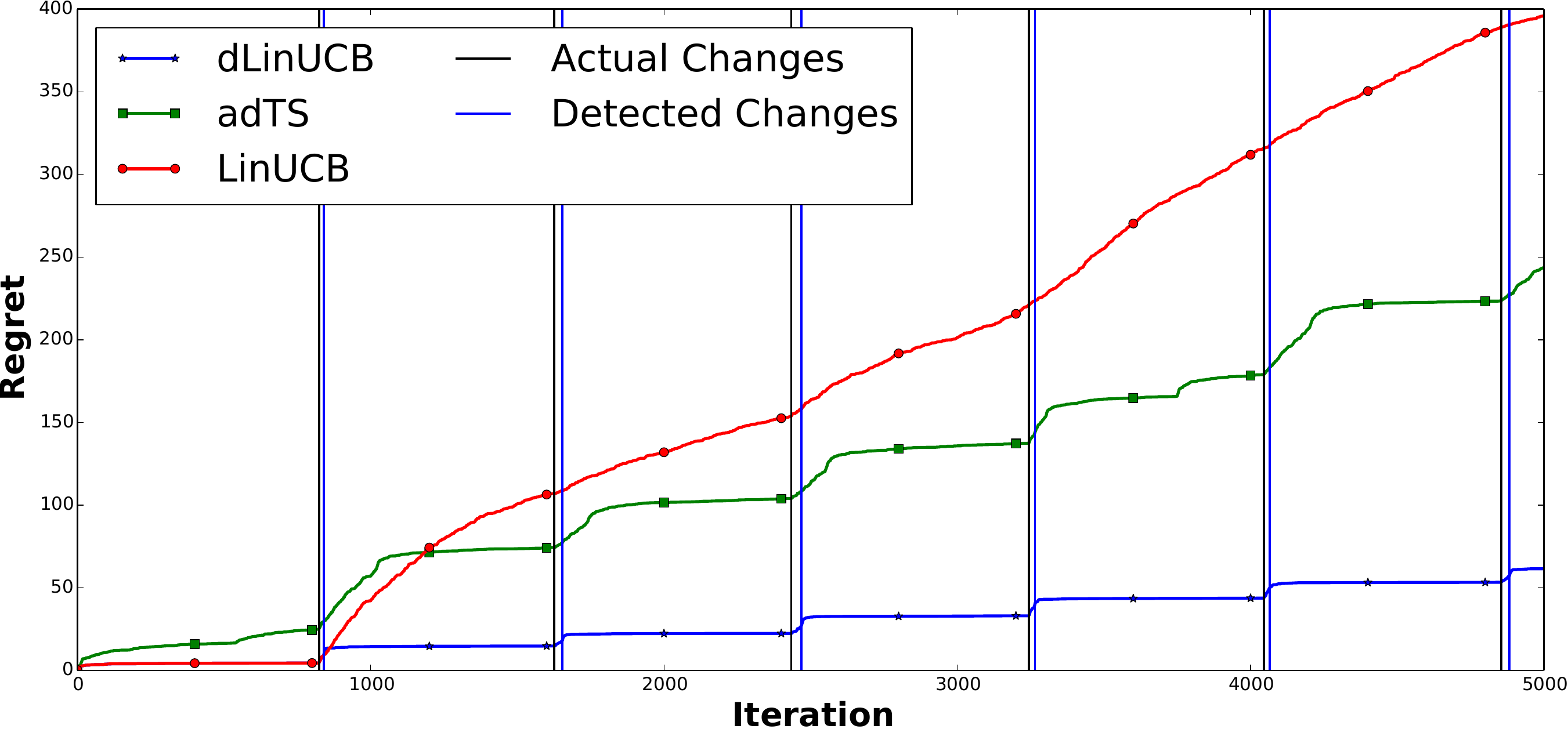}
\vspace{-3mm}
\caption{Results from simulation.} \label{fig:simulation}
\vspace{-4mm}
\end{figure}

\subsection{Experiments on Yahoo! Today Module}
\begin{figure*}[t!h]
\centering
\begin{tabular}{c c c}
\includegraphics[width=5.4cm]{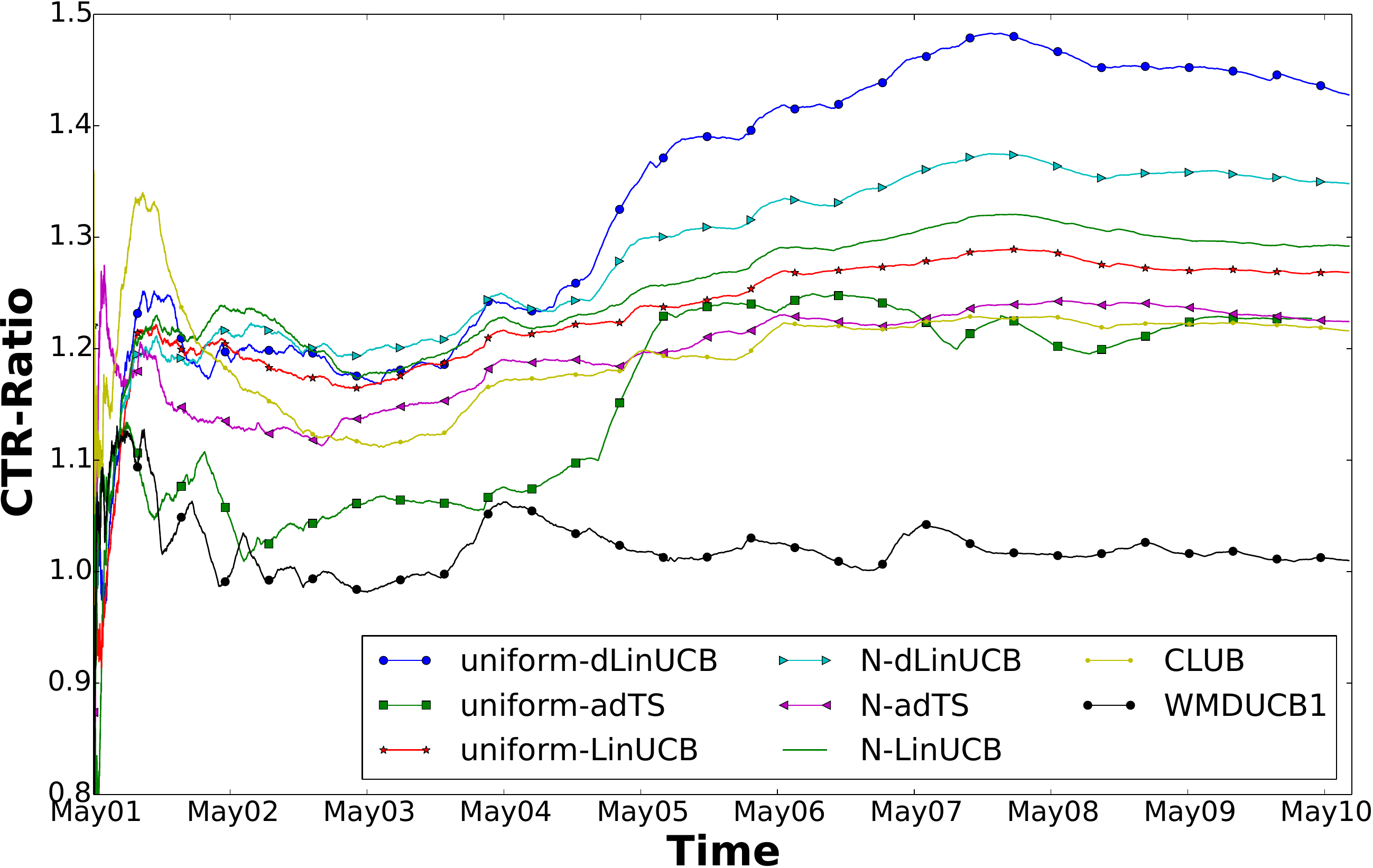} 
 & \includegraphics[width=5.3cm]{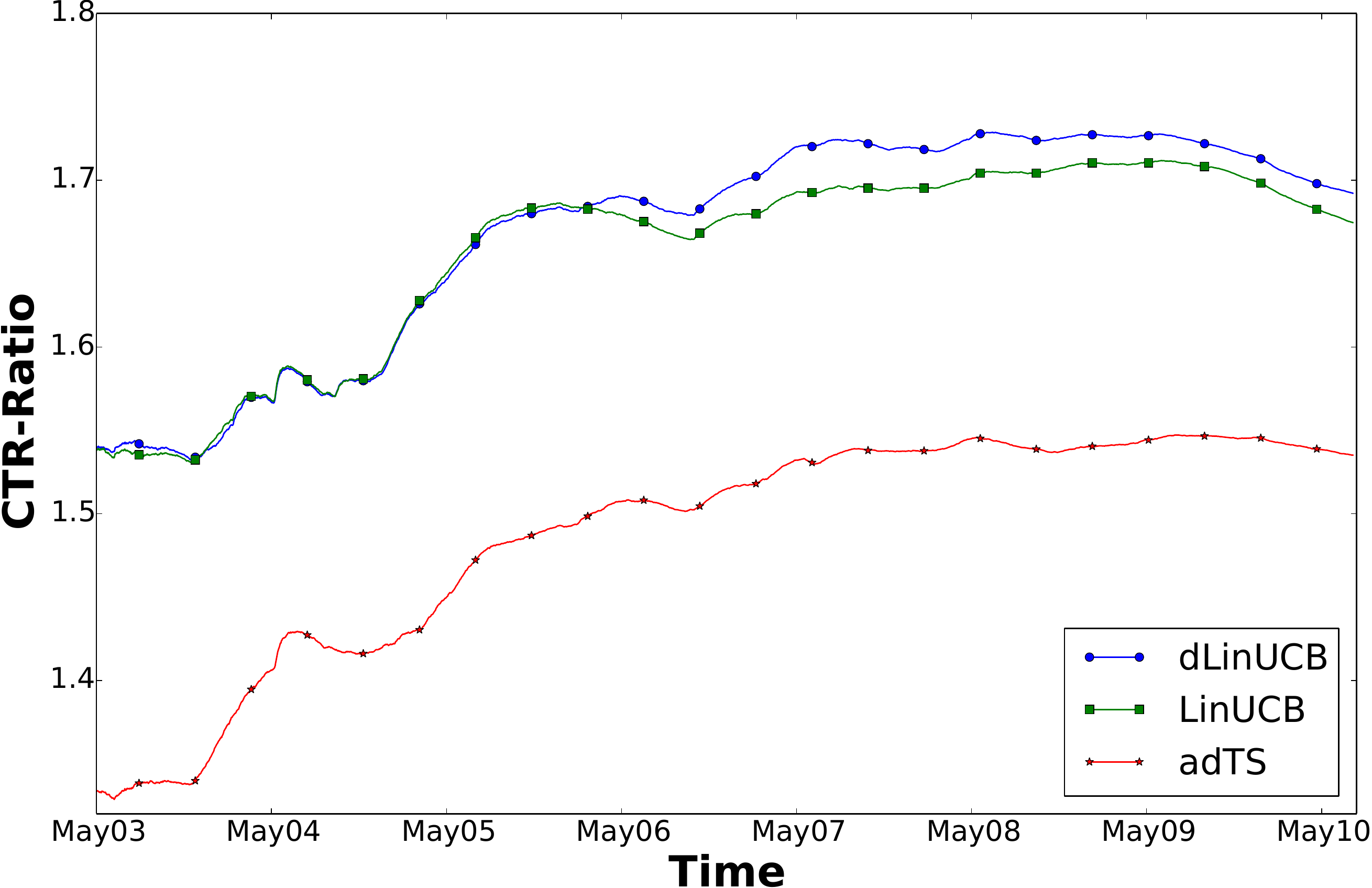}
 &
 \includegraphics[width=5.4cm]{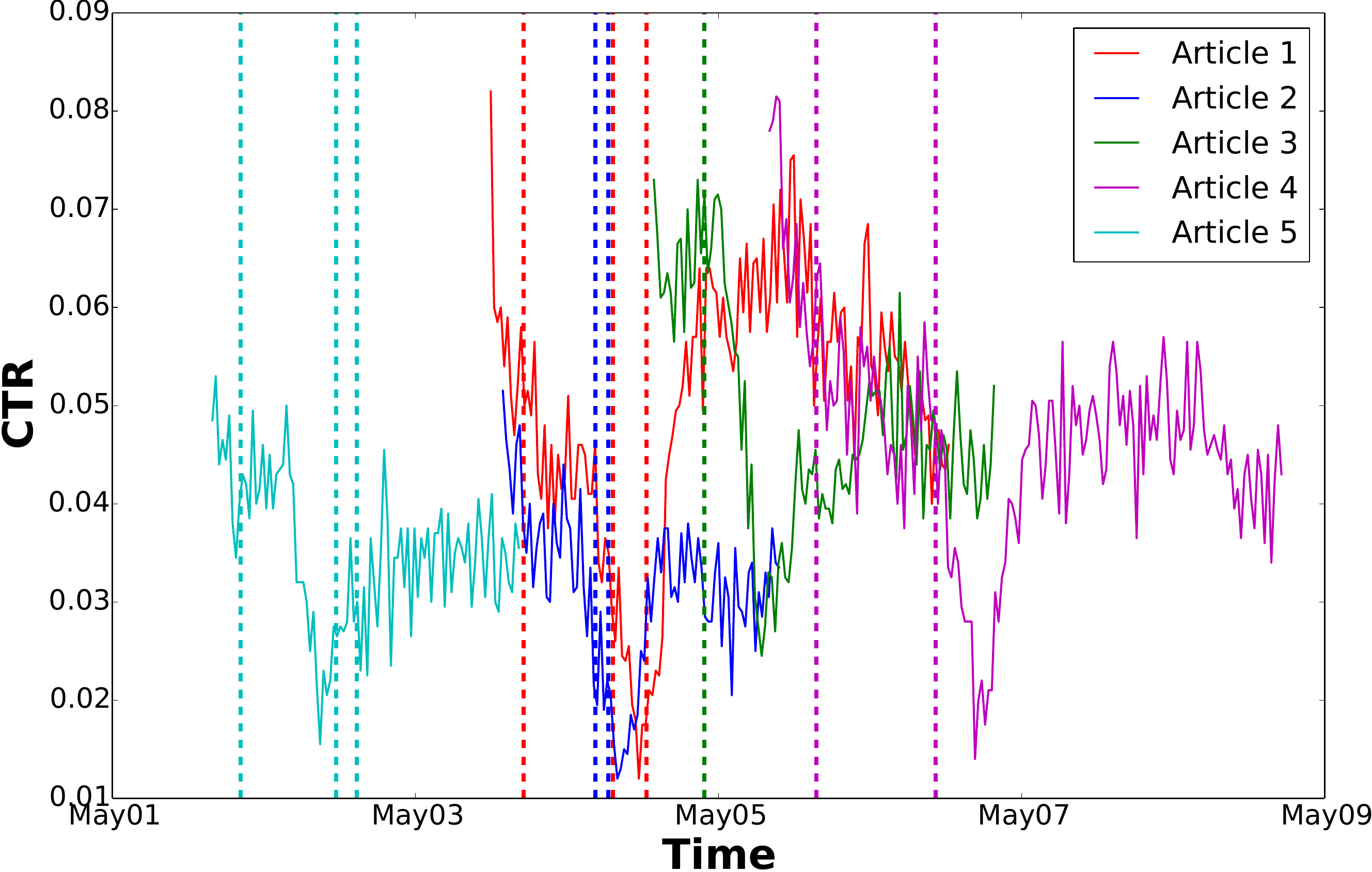}\\
(a) Bandit models on the user side   & (b) Bandit models on the article side   & (c) Detected changes on sample articles
\end{tabular}
\vspace{-3mm}
\caption{Performance comparison in Yahoo! Today Module.} \label{fig:yahoo}
\vspace{-4mm}
\end{figure*}

We compared all the algorithms on the large-scale clickstream dataset made available by the Yahoo Webscope program. 
This dataset contains 45,811,883 user visits to Yahoo Today Module in a ten-day period in May 2009. For each visit, both the user and each of the 10 candidate articles are associated with a feature vector of six dimensions (including a constant bias term) \cite{LinUCB}. 
In the news recommendation problem, it is generally believed that users' interests on news articles change over time; and it is confirmed in this large-scale dataset by our quantitative analysis. To illustrate our observations, we randomly sampled 5 articles and reported their real-time click-through-rate (CTR) in Figure \ref{fig:yahoo} (c), where each point is the average CTR over 2000 observations. Clearly, there are dramatic changes in those articles' popularity over time. For example, article 1's CTR kept decreasing after its debut, then increased in the next two days, and dropped eventually. Any recommendation algorithm failing to recognize such changes would suffer from a sub-optimal recommendation quality over time.  

The unbiased offline evaluation protocol proposed
in \cite{Li:2011:UOE:1935826.1935878} is used to compare different algorithms. CTR is used as the performance metric of all bandit algorithms. Following the same evaluation principle used in \cite{LinUCB}, we
normalized the resulting CTR from different algorithms by the corresponding
logged random strategy's CTR.
We tested two different settings on this dataset based on where to place the bandit model for reward estimation.

The first setting is to build bandit models for users, i.e., attaching $\btheta$ on the user side to learn users' preferences over articles. We included a non-personalized variant and a personalized variant of all the contextual bandit algorithms. In the non-personalized variant,  the bandit parameters are shared across all users, and thus the detected changes are synchronized across users. We name the resulting algorithms as uniform-LinUCB, uniform-adTS, and uniform-dLinUCB. In the personalized variant, each individual user is associated with an independent bandit parameter $\btheta_{u}$, and the change is only about him/herself. Since this dataset does not provide  user identities, we followed \cite{wu2016contextual} to cluster users into $N$ user groups and assume those in the same group share the same bandit parameter. We name the resulting algorithms as N-LinUCB, N-adTS and N-dLinUCB. To make the comparison more competitive, we also include a recently introduced collaborative bandit algorithm CLUB \cite{gentile2014online}, which combines collaborative filtering with bandit learning. 

From Figure \ref{fig:yahoo} (a), we can find that both the personalized and non-personalized variants of dLinUCB achieved significant improvement compared with all baselines. It is worth noticing that uniform-dLinUCB obtained around $50\%$ improvement against uniform-LinUCB, $15\%$ against N-LinUCB, and $25\%$ against CLUB. Clearly assuming all the users share the same preference over the recommendation candidates is very restrictive, which is confirmed by the improved performance from the personalized version over the non-personalized version of all bandit algorithms. Because dLinUCB maintains multiple slave models concurrently, each slave model is able to cover preference in a subgroup of users, i.e., achieving personalization automatically. We looked into those created slave models and found they closely correlated with the similarity between user features in different groups created by \cite{wu2016contextual}, although such external grouping was not disclosed to uniform-dLinUCB.  Although adTS and WMDUCB1 can also detect changes, its slow detection and reaction to the changes made it even worse than LinUCB on this dataset. Meta-Bandit is sensitive to its hyper-parameters and performed similarly to WMDUCB1, so that we excluded it from this comparison.

The second setting is to build bandit models for each article, i.e., attaching $\btheta$ on the article side to learn its popularity over time. Based on our quantitative analysis in the data set, we found that articles with short lifespans tend to have constant popularity. To emphasize the non-stationarity in this problem, we removed articles which existed less than 18 hours, and report the resulting performance in Figure \ref{fig:yahoo} (b). We can find that dLinUCB performed comparably to LinUCB at the beginning, while the adTS baselines failed to recognize the popularity of those articles from the beginning, as the popularity of most articles did not change immediately. In the second half of this period, however, we can clearly realize the improvement from dLinUCB. To understand what kind of changes dLinUCB recognized in this data set, we plot the detected changes of five randomly selected  articles in Figure \ref{fig:yahoo} (c), in which dotted vertical lines are our detected change points on corresponding articles. As we can find in most articles the critical changes of ground-truth CTR can be accurately recognized. For example, article 1 and article 2 at around May 4, and article 3 at around May 5. Unfortunately, we do not have any detailed information about these articles to verify the changes; otherwise it would be interesting to correspond these detected changes to real-world events. In Figure \ref{fig:yahoo} (b), we excluded the context-free bandit algorithms because they performed much worse and complicate the plots.

\subsection{Experiments on LastFM \& Delicious}
\begin{figure*}[t!h]
\centering
\begin{tabular}{c c c c}
\includegraphics[width=4.29cm]{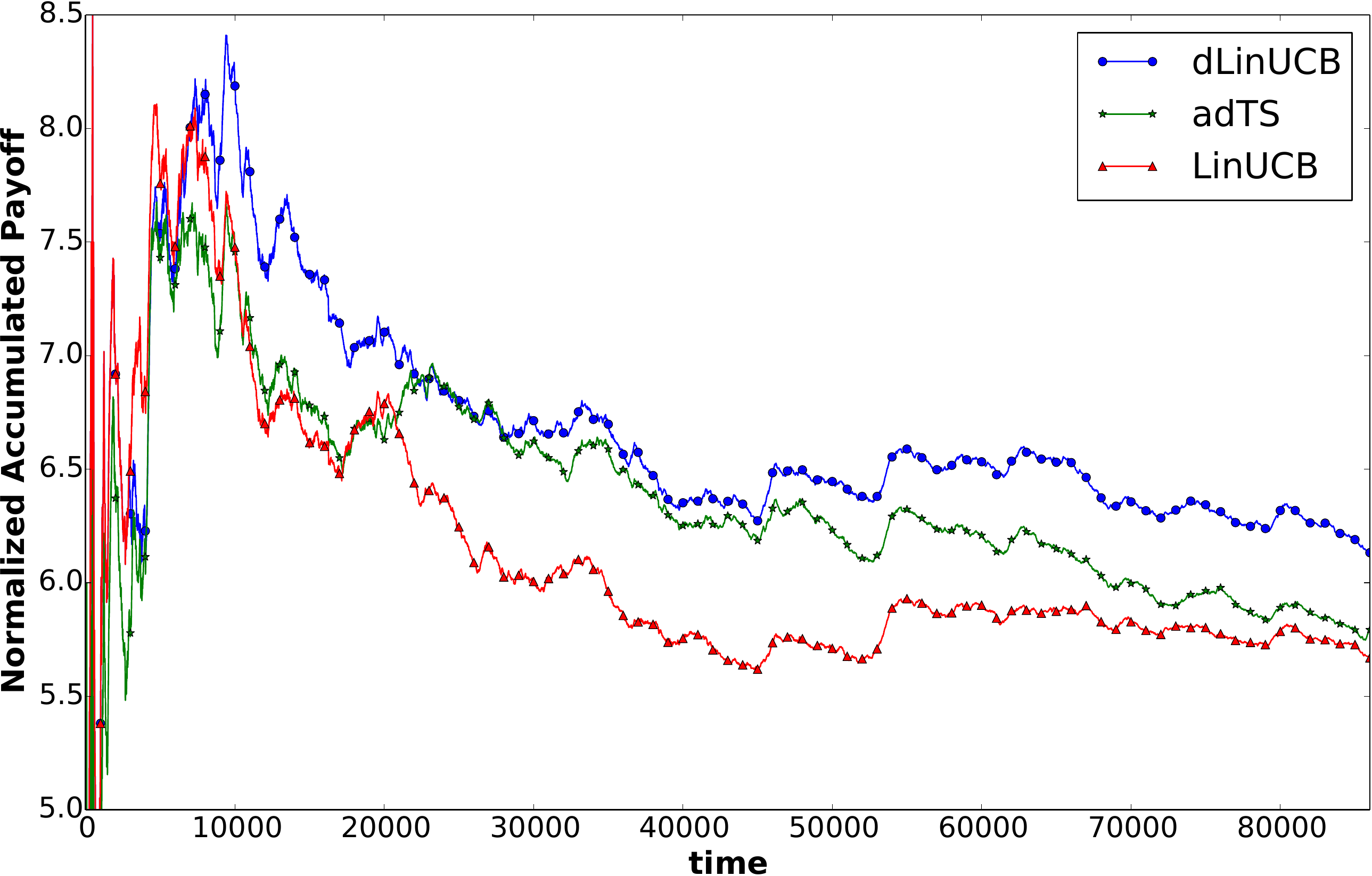} &
\includegraphics[width=4.45cm]{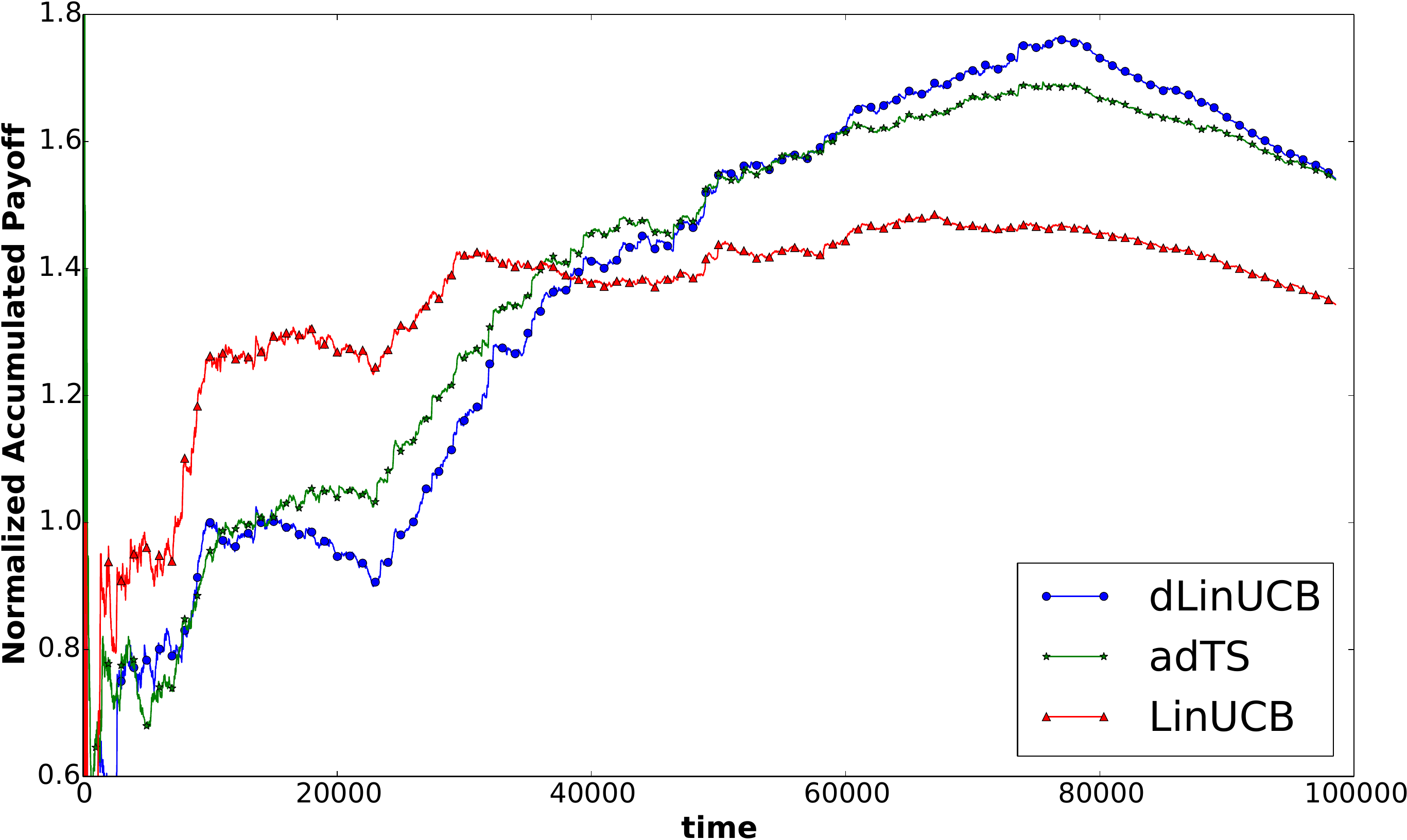} &
\includegraphics[width=4.29cm]{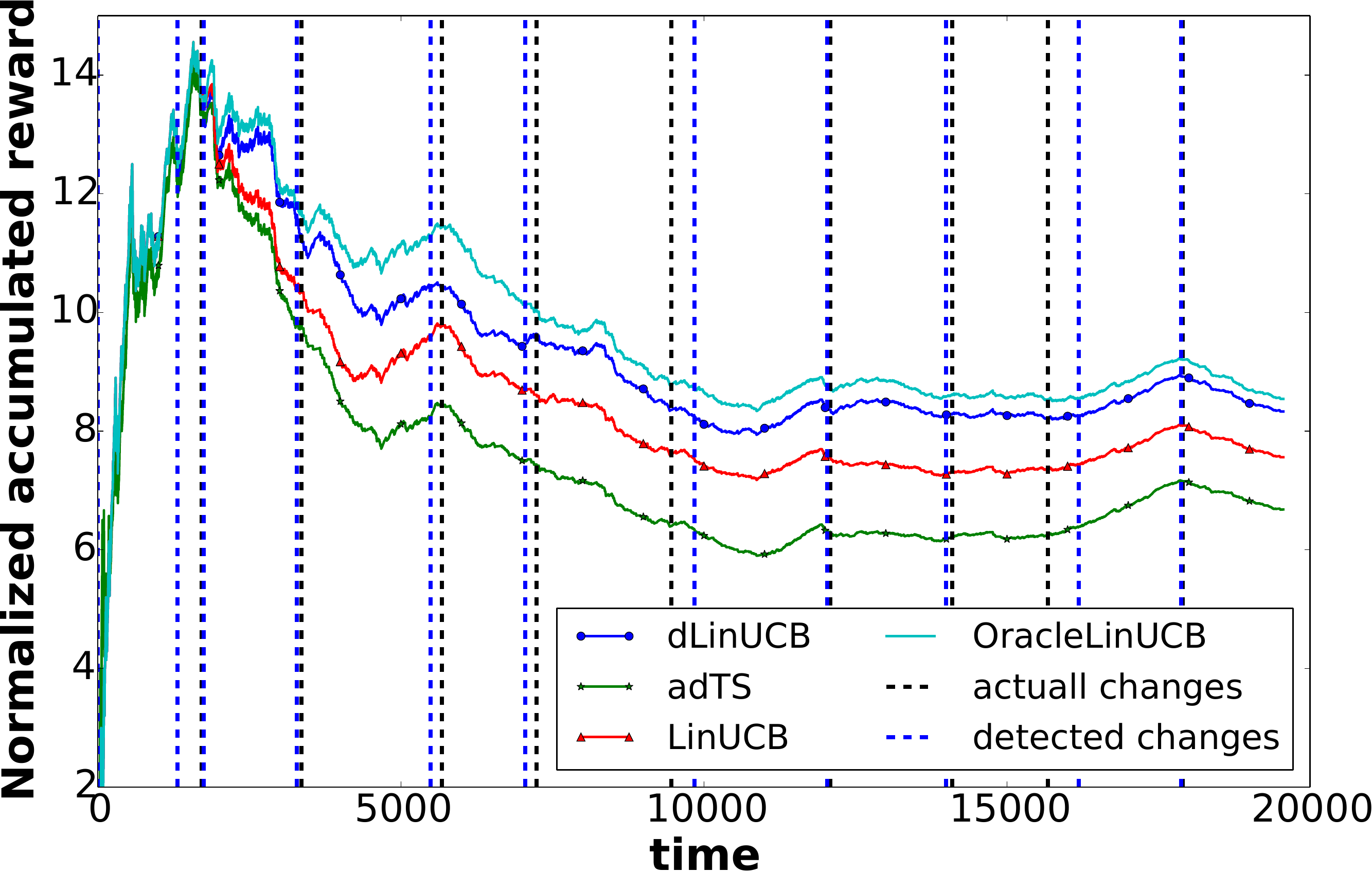} &
\includegraphics[width=4.29cm]{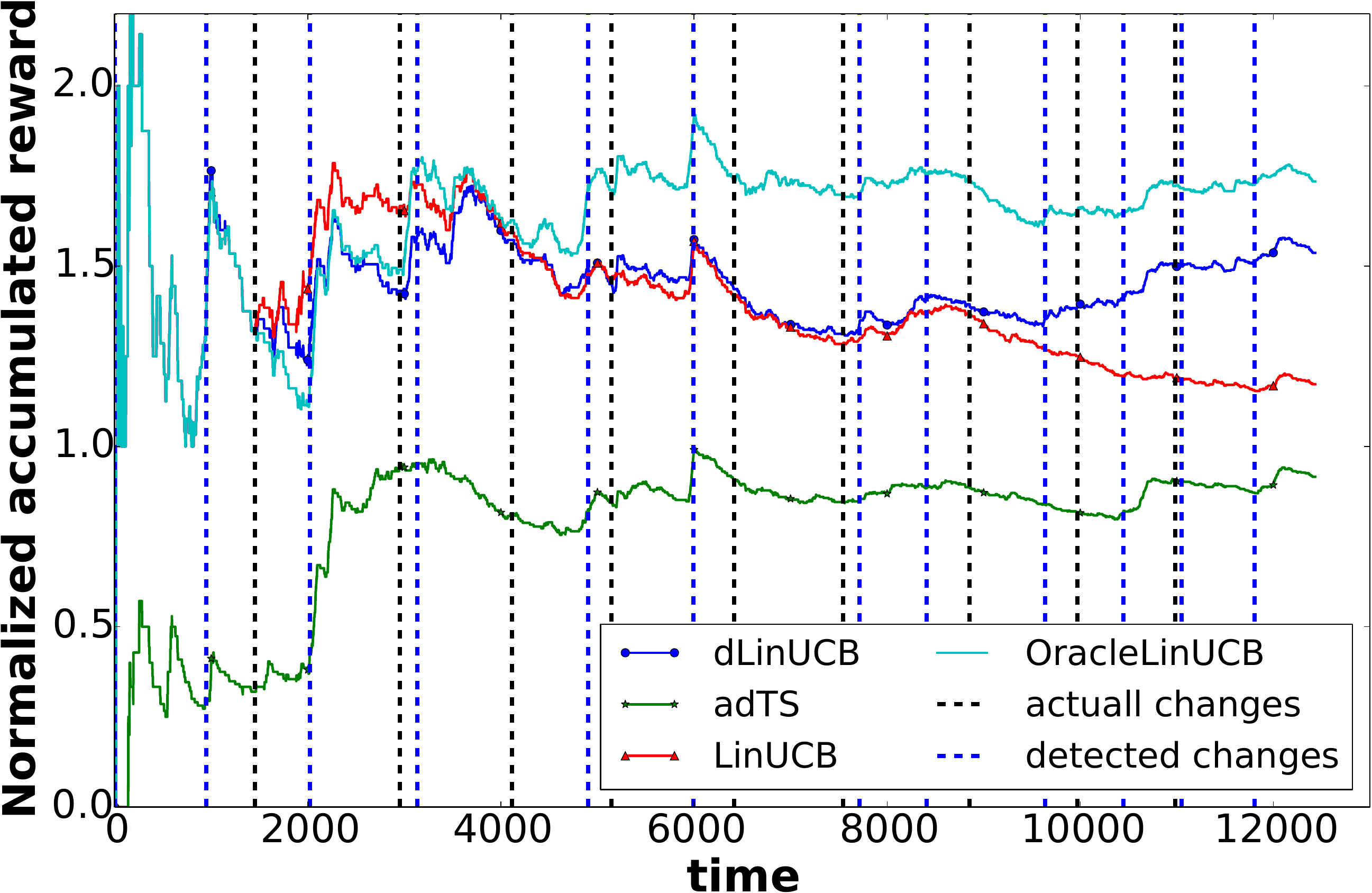} \\
\small (a) Normalized reward on LastFM  & \small (b) Normalized reward on Delicious   & \small (c) Cluster detection on LastFM  & \small (d) Cluster detection on Delicious
\normalsize
\end{tabular}
\vspace{-3mm}
\caption{Performance comparison in LastFM \& Delicious.} \label{fig:lastfm_delicious}
\vspace{-3mm}
\end{figure*}

The LastFM dataset is extracted from the music streaming service
Last.fm, and the Delicious dataset is extracted from the social bookmark sharing service Delicious. They were made availalbe on the HetRec 2011 workshop. 
The LastFM dataset contains 1892 users and 17632 items (artists). We treat the `listened artists' in each user as positive feedback. The Delicious dataset contains 1861 users and 69226 items (URLs). We treat the bookmarked URLs in each user as positive feedback. Following the settings in \cite{Gang}, we pre-processed these two datasets in order to fit them into the contextual bandit setting. Firstly, we used all tags associated with an item to create a TF-IDF feature vector to represent it. Then we used PCA to reduce the dimensionality of the feature vectors and retained the first 25 principle components to construct the context vectors, i.e., $d = 25$. We fixed the size of candidate arm pool to $K=25$; for a particular user $u$, we randomly picked one item from his/her nonzero reward items, and randomly picked the other 24 from those zero reward items. We followed \cite{hartland:hal-00113668} to simulate a non-stationary environment: we ordered observations chronologically inside each user, and built a single hybrid user by merging different users. Hence, the boundary between two consecutive batches of observations from two original users is treated as the preference change of the hybrid user. 

Normalized rewards on these two datasets are reported in Figure \ref{fig:lastfm_delicious} (a) \& (b). dLinUCB outperformed both LinUCB and adTS on LastFM. As Delicious is a much sparser dataset, both adTS and dLinUCB are worse than LinUCB at the beginning; but as more observations become available, they quickly catch up. Since the distribution of items in these two datasets are highly skewed \cite{Gang}, which makes the observations for each item very sparse, the context-free bandits performed very poorly on these two datasets. We therefore chose to exclude the context-free bandit algorithms from all the comparisons on these two datasets in our result report.


Each slave model created for this hybrid user can be understood as serving for a sub-population of users. We qualitatively studied those created slave models to investigate what kind of stationarity they have captured. On the LastFM dataset, each user is associated with a list of tags he/she gave to the artists. The tags are usually descriptive and reflect users' preference on music genres or artist styles. In each slave model, we use all the tags from the users being served by this model to generate a word cloud. Figure \ref{fig:wordcloud} are four representative groups identified on LastFM, which clearly correspond to four different music genres -- rock music, metal music, pop music and hip-hop music. dLinUCB recognizes those meaningful clusters purely from user click feedback. 

The way we simulate the non-stationary environment on these two datasets makes it possible for us to assess how well dLinUCB detects the changes. To ensure result visibility, we decide to report results obtained from user groups (otherwise there will be too many change points to plot). We first clustered all users in both of datasets into user groups according to their social network structure using spectral clustering \cite{Gang}. Then we selected the top 10 user groups according to the number of observations to create the hybrid user. We created a semi-oracle algorithm named as OracleLinUCB, which knows where the boundary is in the environment and resets LinUCB at each change point. The normalized rewards from these two datasets are reported in Figure \ref{fig:lastfm_delicious} (c) \& (d), in which the vertical lines are the actual change points in the environment and the detected points by dLinUCB. Since OracleLinUCB knows where the change is ahead of time, its performance can be seen as optimal. On LastMF, the observations are denser per user group, so that dLinUCB can almost always correctly identify the changes and achieve quite close performance to this oracle. But on Delicious, the sparse observations make it much harder for change detection; and more early and late detection happened in dLinUCB.


\begin{figure}[t]
\centering
\begin{tabular}{c c}
\includegraphics[width=0.41\linewidth]{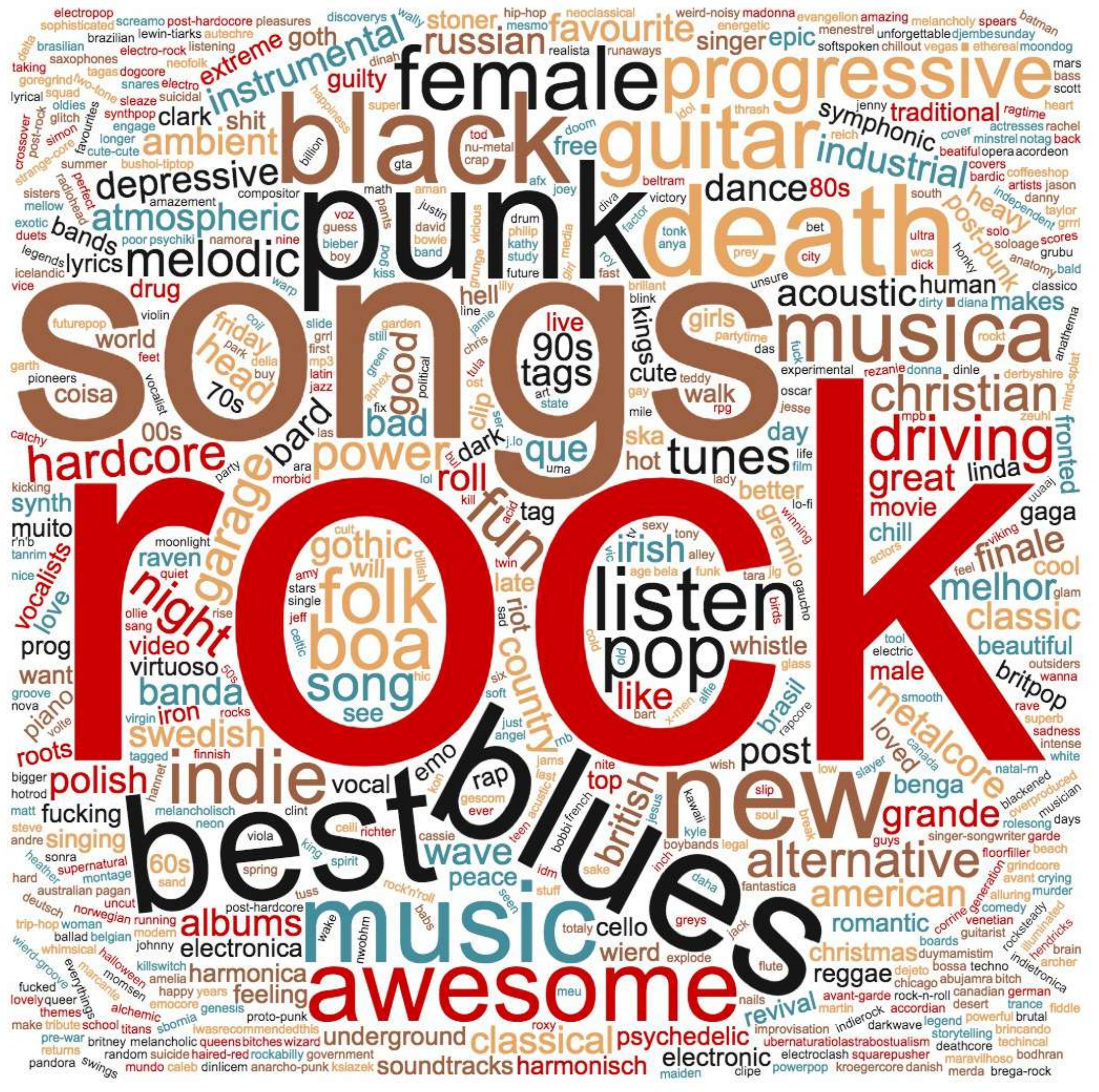} &
\includegraphics[width=0.41\linewidth]{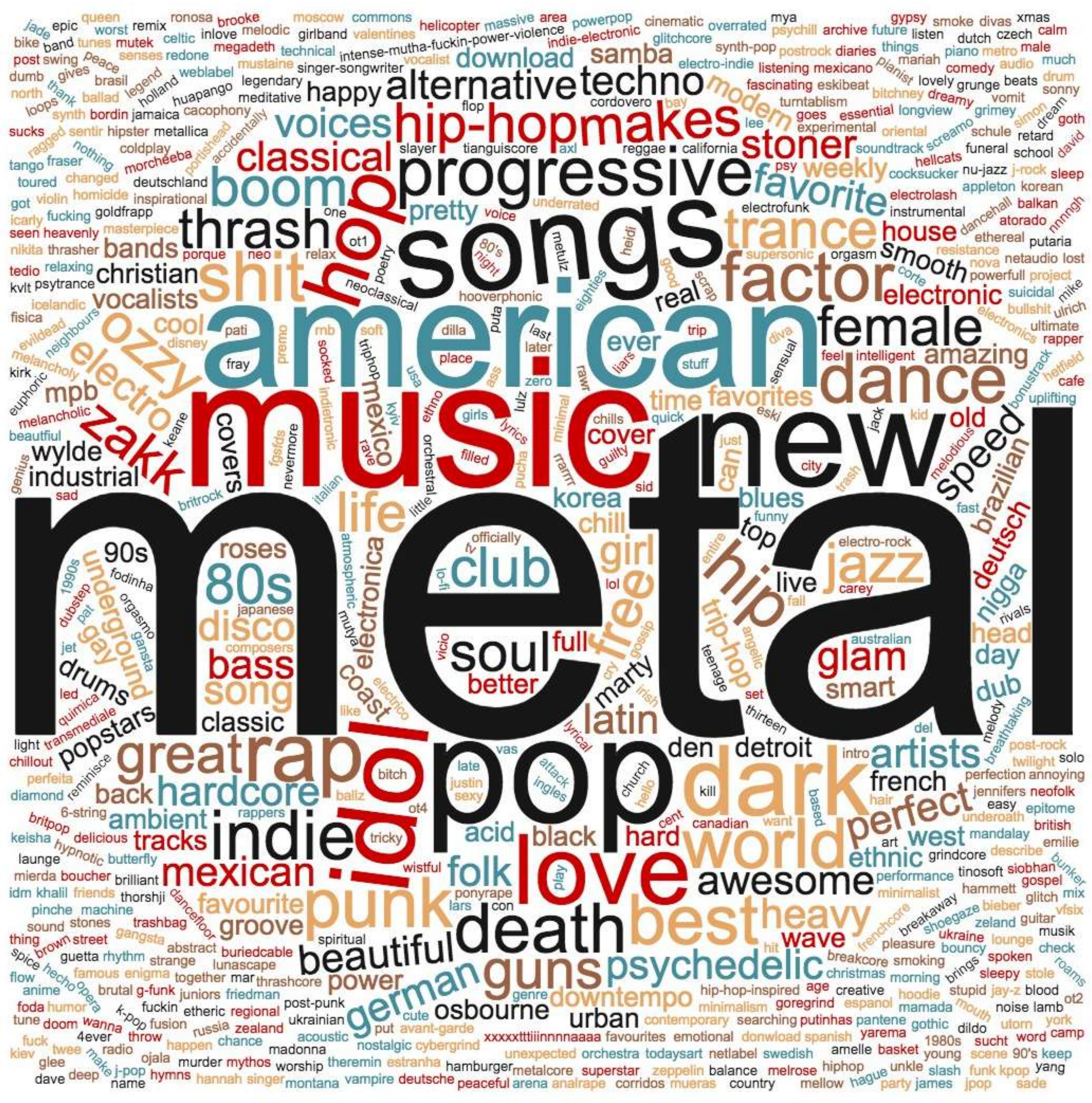} \\
\includegraphics[width=0.41\linewidth]{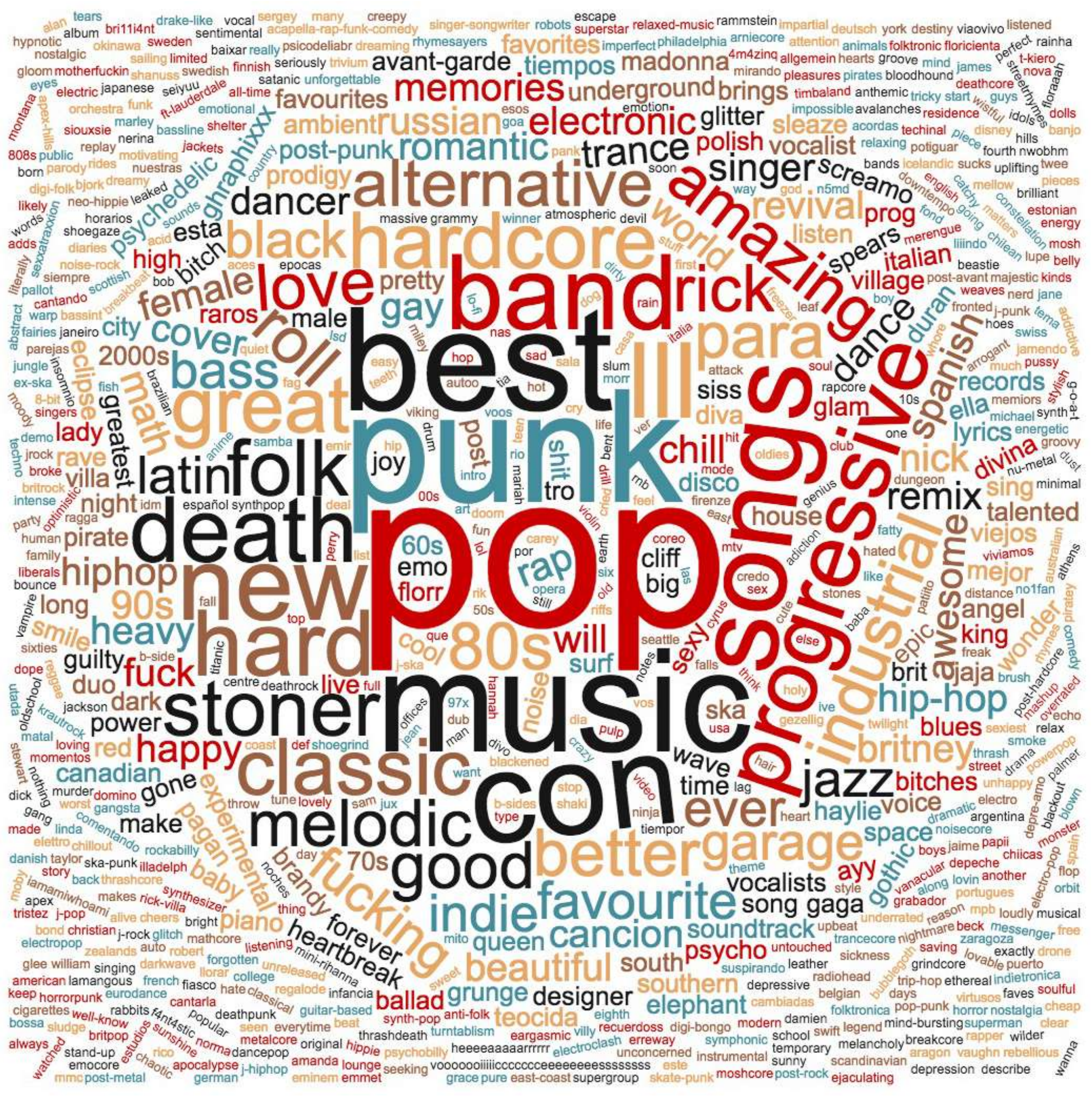} &
\includegraphics[width=0.41\linewidth]{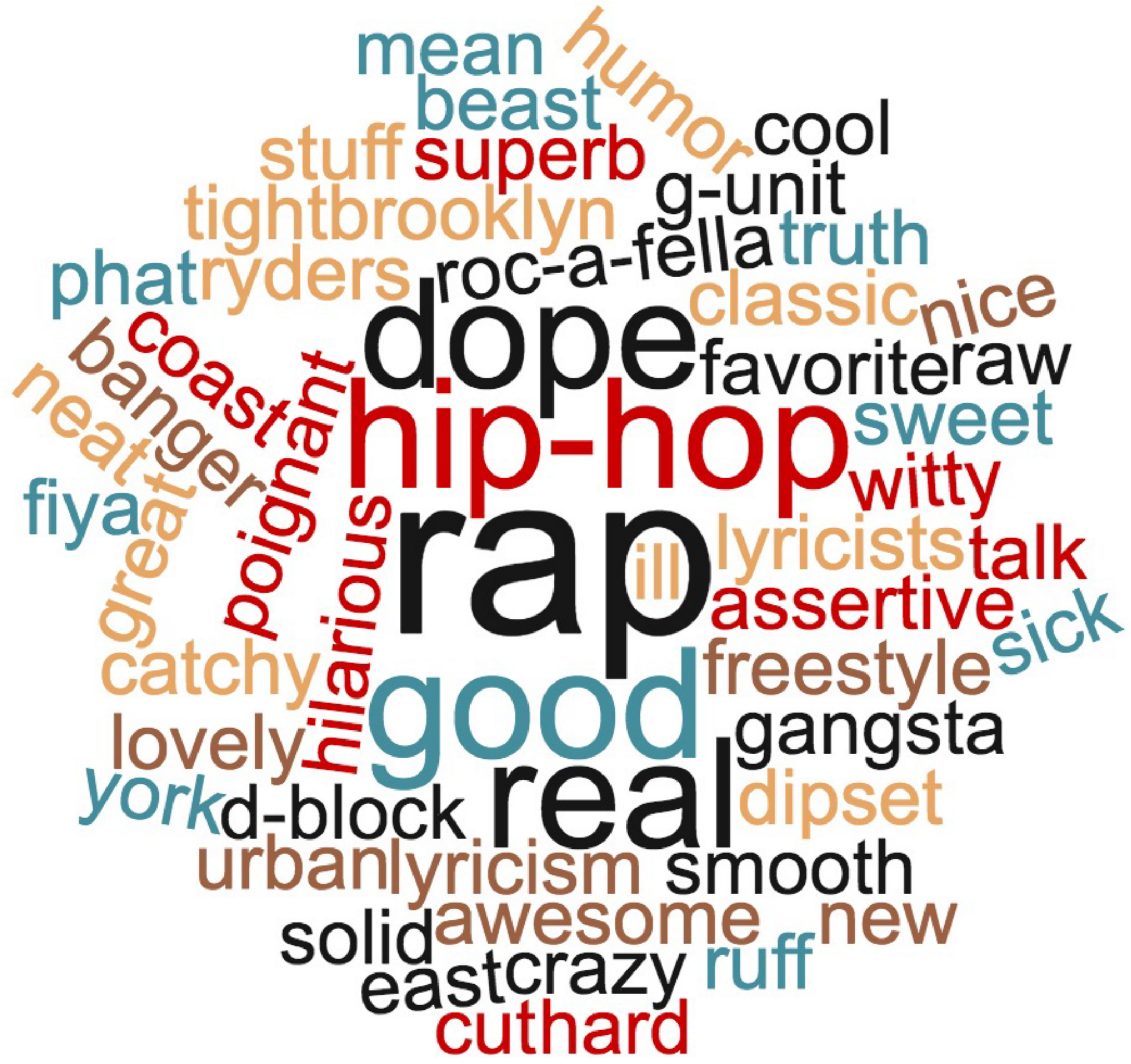} \\
\end{tabular}
\vspace{-4mm}
\caption{Word cloud of tags from four identified user groups in dLinUCB on LastFM dataset.} 
\label{fig:wordcloud}
\vspace{-5mm}
\end{figure}

\section{Conclusions \& Future Work}

In this paper, we develop a contextual bandit model dLinUCB for a piecewise stationary environment, which is very common in many important real-world applications but insufficiently investigated in existing works. By maintaining multiple contextual bandit models and tracking their reward estimation quality over time, dLinUCB adaptively updates its strategy for interacting with a changing environment. We rigorously prove an $O(\Gamma_T \sqrt{S_T}\ln S_T)$ upper regret bound, which is arguably the tightest upper regret bound any algorithm can achieve in such an environment without further assumption about the environment. Extensive experimentation in simulation and three real-world datasets verified the effectiveness and the reliability of our proposed method. 

As our future work, we are interested in extending dLinUCB to a continuously changing environment, such as Brownian motion, where reasonable approximation has to be made as a model becomes out of date right after it has been created. Right now, when serving for multiple users, dLinUCB treats them as identical or totally independent. As existing works have shed light on collaborative bandit learning \cite{wu2016contextual, FactorUCB, pmlr-v70-gentile17a}, it is meaningful to study non-stationary bandits in a collaborative environment. Last but not least, currently the master bandit model in dLinUCB does not utilize the available context information for `badness' estimation. It is necessary to incorporate such information to improve the change detection accuracy, which would lead to a further reduced regret.  


\section{Acknowledgments}
We thank the anonymous reviewers for their insightful comments. This work was supported in part by National Science Foundation Grant IIS-1553568 and IIS-1618948.

\bibliographystyle{ACM-Reference-Format}


\section{Appendix}
\subsection{Additional Theorems}
If the training instances $\{(\bx_i, r_i)\}_{i \in {\cI_{m,t}}}$ in a linear bandit model come from multiple distributions/environments, we separate the training instances in $\cI_{m,t}$ into two sets $\mathcal{H}_{m,t}$ and $\tilde {\mathcal{H}}_{m,t}$ so that instances from $\mathcal{H}_{m,t}$ are from the target stationary distribution, while instances in $\tilde {\mathcal{H}}_{m,t}$ are not. In this case, we provide the confidence bound for the reward estimation in Theorem \ref{theorem:LinUCBConfidenceBound_withcontamination}.

\begin{theorem} [LinUCB with contamination] \label{theorem:LinUCBConfidenceBound_withcontamination}
In LinUCB with a contaminated instance set  $\tilde {\mathcal{H}}_{m,t}$,  with probability at least $1-\delta_1$, we have
$\lvert \hat r_t(m) -  \bbE[r_t] \rvert     \leq \tilde B_t$, where $\tilde B_t(m,a) = \tilde \alpha_t \lVert \bx_{a_t} \rVert_{\bA_{t-1}^{-1}} $, $\tilde \alpha_t = \sigma^2 \sqrt{d \ln(1+\frac{|\cI_{m,t}|}{\lambda \delta_1})} + \sqrt{\lambda} + C_t(m)$, and $C_t = \sum_{i\in \tilde{ \mathcal{H}}_{m,t}}(\bx_{a_i} (\btheta_i^* - \btheta_{t_c}^*))$.
\end{theorem}
Comparing $\tilde{B}_t(m,a)$ with $B_t(m,a)$, we can see that when the reward deviation of an arm (the $1-\rho$ portion of arms that do not satisfy Eq \eqref{assumtion:changeAssumption} in Assumption \ref{assumtion:changeAssumption}) is small with $(1-\rho) \Delta_{\text{small}} \leq \big(\sigma^2 \sqrt{d \ln(1+\frac{|\cI_{m,t}|}{\lambda \delta_1})} \big)/\big( \tilde{ \mathcal{H}}_{m,t}\big)$, the same confidence bound scaling can be achieved. 

\begin{theorem}[Chernoff Bound]\label{theorem:chernoffBound}
Let $Z_1, Z_2,...,Z_n$ be random variables on $\mathbb{R}$ such that $a \leq Z_i \leq b$. Define $W_n = \sum_{i=1}^n Z_i $, for all $c>0$ we have,
\small
\begin{equation*}
\mathbb{P}(\lvert W_n - \mathbb{E}(W_n) \rvert > c\mathbb{E}(W_n) ) \leq 2 \exp{\Big(\frac{-2c^2\mathbb{E}(W_n)^2}{n (b-a)^2}\Big)}
\end{equation*}
\normalsize
\end{theorem}

\subsection{Proof of Theorems and Lemmas}

\begin{proof}[Proof sketch of Theorem \ref{theorem:slave_CB} and Theorem \ref{theorem:LinUCBConfidenceBound_withcontamination}]
The proof of Eq \eqref{eq:slave_residual_bound} in Theorem \ref{theorem:slave_CB} and Theorem \ref{theorem:LinUCBConfidenceBound_withcontamination} are mainly based on the proof of Theorem 2 in \cite{Improved_Algorithm} and the concentration property of Gaussian noise.
\end{proof}

\begin{proof}[Proof of Lemma \ref{lemma:early_detection}]
According to Chernoff Bound, we have $P(\hat e_t(m) \leq \delta_1 + \frac{\ln (1/\delta_2)}{2\tilde{\tau}(m)}) \geq 1-\delta_2$, which concludes the proof. 
\end{proof}

\begin{proof}[Proof of Lemma \ref{lemma:late_detection}]
At time $i \geq t_{c_{j+1}}$, which means the environment has already changed from $\btheta^*_{c_{j}}$ to $\btheta^*_{c_{j+1}}$, we have,
\small
\begin{align} \label{eq:badness_afterchange}
    & \bbP(e_i(m) = 1)  = \bbP(\lvert \hat r_i -r_i \rvert > B_i(m,a_i) + \epsilon ) \\ \nonumber
     = &\bbP(\lvert ( \bx_i^{\mt} \hat \btheta_i  - \bx_i^{\mt} \btheta^*_{t_{c_j}} -\eta_i) + ( \bx_i^{\mt} \btheta^*_{t_c}- \bx_i^{\mt} \btheta_i^*   ) \rvert > B_i(m,a_i) + \epsilon ) \nonumber
\end{align}
\begin{align}
      \geq & \bbP(\lvert  \bx_i^{\mt} \hat \btheta_i  - \bx_i^{\mt} \btheta^*_{t_{c_j}} -\eta_i \rvert \leq  \tilde{B}_{i}(m,a_i) + \epsilon) \nonumber \\ 
     & \times \bbP( \lvert \bx_i^{\mt} \btheta^*_{t_{c_j}}- \bx_i^{\mt} \btheta_i^*   \rvert > \tilde{B}_{i}(m,a_i) + B_{i}(m,a_i) + 2\epsilon )\nonumber
\end{align}
\normalsize
According to Theorem \ref{theorem:LinUCBConfidenceBound_withcontamination}, we have $\bbP\big(\lvert  \bx_i^{\mt} \hat \btheta_i  - \bx_i^{\mt} \btheta^*_{t_{c_j}} -\eta_i \rvert \leq  \tilde{B}_{i}(m,a_i) + \epsilon\big)  \geq 1-\delta_1$.  Define $U_{c_j}$ as the upper bound of $\tilde B_{i}(m,a_i)+ B_{t_c}(m,a_i) + 2\epsilon $. If the change gap $\Delta_{c_j}$ satisfies $\Delta_{c_j} \geq U_{c_j} $, we have $\bbP(e_i(m) = 1) \geq \rho (1-\delta_1) $. 

Next, we will prove that $\Delta_{c_j} \geq U_{c_j} $ can be achieved by a properly set $\delta_1$. Similar as the proof in Step 2 of Theorem \ref{theorem:regretBoundOnly1}, where we bound $k_c$, we have with a high probability that $\tilde B_{i}(m,a_i) + B_{t_{c_j}}(m,a_i)  + 2\epsilon  \leq  2\epsilon + 2\sqrt{\lambda} +  \frac{2}{\sqrt{\lambda}} S_{c_j}\rho \big(1- \rho (1-\delta_1)\big) = U_{c_j} $. When $\Delta_{c_j} > 2\sqrt{\lambda} + 2\epsilon $, $S_{\text{min}} > \frac{\sqrt{\lambda}}{2\rho}(\Delta_{c_j} - 2\sqrt{\lambda} - 2\epsilon )$ and $\delta_1 \leq 1- \frac{1}{\rho}(1-\frac{\sqrt{\lambda}}{2S_{\text{min}}\rho}\big(\Delta_{c_j} - 2\sqrt{\lambda} - 2\epsilon )\big)$, $\Delta_{c_j} > U_{c_j}$ can be achieved. 


Eq \eqref{eq:badness_afterchange} indicates when the environment has changed for a slave model $m$, with a high probability of $e_i(m) =1$ and slave model $m$ will not be updated, which avoids possible contamination in $m$. According to the concentration inequality in Theorem \ref{theorem:chernoffBound}, with a probability at least $1-\delta_2$, we have,
\small
\begin{equation*}
\sum_{i=t-\tau}^{t} e_i(m) \geq  \bbE[\sum_{i=t-\tau}^{t} e_i(m)] - \sqrt{\frac{\tau}{2} \ln\frac{1}{\delta_2}} \geq \rho (1-\delta_1)\tau - \sqrt{\frac{ \tau}{2} \ln\frac{1}{\delta_2}}
\end{equation*}
\normalsize
With simple rewriting, we have when $\tau \geq \frac{2\ln \frac{2}{\delta_2 }}{ (\rho (1-\delta_1) - \delta_1)^2}$, $\rho (1-\delta_1)\tau - \sqrt{\frac{ \tau}{2} \ln\frac{1}{\delta_2}} \geq  \delta_1 \tau +  \sqrt{\frac{ \tau}{2} \ln\frac{1}{\delta_2}}$, which means that with a probability at least $1-\delta_2$, $\hat e_t(m) = \frac{ \sum_{i=t-\tau}^{t} e_i}{\tau} \geq  \delta_1 +  \sqrt{\frac{1}{2\tau} \ln \frac{1}{\delta_2}}$
\end{proof}

\begin{proof}[Proof of Lemma \ref{lemma:N_bound}]

Under the model selection strategy in line
4 of Algorithm \ref{alg}, using similar proof techniques as in Theorem 1 of \cite{UCB1}, we have
\begin{equation} \label{eq:N_bound}
\small
 \tilde  N_{c_j}(m)  \leq l + \sum_{t={t_{c_j}}}^{\infty} \sum_{s=1}^{t-1} \sum_{s_i = l}^{t-1}  \mathds{1} \{\hat e_{s_i}(m) - d_{s_i}(m) \leq  \hat e_{s}(m_{c_j}^*) - d_{s}(m_{c_j}^*)  \}
  \normalsize
\end{equation}
in which $l = \ceil[\big]{(8 \log S_{c_j})/g_{m,m_{c_j}}^2}$.  $\hat e_{s_i}(m) - d_{s_i}(m) \leq  \hat e_{s}(m_{c_j}^*) - d_{s}(m_{c_j}^*)$ implies that at least one of the three following inequalities must hold,
\small
 \begin{equation} \label{eq:case1}
 \hat e_i(m_{c_j}^*) \geq \bbE[e(m_{c_j}^*)] + d_i(m_{c_j}^*)
 \end{equation}
  \begin{equation}\label{eq:case2}
 \hat e_i(m) \leq \bbE[e(m)] - d_i(m)
 \end{equation}
  \begin{equation}\label{eq:case3}
 \bbE [e_i(m)] - 2d_i(m) \leq \bbE[e(m_{c_j}^*)] 
 \end{equation}
\normalsize
Intuitively, Eq \eqref{eq:case1}, \eqref{eq:case2} and \eqref{eq:case3} correspond to the following three cases respectively. First, the expected `badness' of the optimal slave model $m_{c_j}^*$ is substantially over-estimated. Second, the expected `badness' of the slave model $m$ is substantially under-estimated. Third, the expected `badness' of the two slave models $m_{c_j}^*$ and $m$ are very close to each other.  
 
According to Theorem \ref{theorem:chernoffBound}, we have $\bbP \big(\hat e_i(m_{c_j}^*) \geq \bbE[e(m_{c_j}^*)] + d_i(m_{c_j}^*) \big) \leq \delta_2  $ and $\bbP \big(\hat e_i(m) \leq \bbE[e(m)] - d_i(m) \big) \leq \delta_2 $. For $s_i \geq l$, we have $\bbE [e(m)] -  \bbE[e(m_{c_j}^*)] - 2d_i(m) \geq 0$, which means the case in Eq \eqref{eq:case3} almost never occurs. Substituting the probability for Eq \eqref{eq:case1}, \eqref{eq:case2} and \eqref{eq:case3} into Eq \eqref{eq:N_bound}, and when $\delta_2 = t^{-4}$, 
\small
\begin{equation*}
\tilde  N_{c_j}(m)  \leq \ceil[\big]{(8 \ln S_{c_j})/g_{m,m_{c_j}^*}^2} + \sum_{t=1}^{\infty} \sum_{s=1}^{t-1} \sum_{s_i = l}^{t-1} 2 t^{-4}  \leq (8 \ln S_{c_j})/g_{m,m_{c_j}^*}^2 + 1 + \frac{\pi^2}{3}  
\end{equation*}
\normalsize
which finishes the proof.
\end{proof}

\end{document}